\documentclass{article}
\usepackage{iclr2027_conference,times}
\usepackage[T1]{fontenc}
\usepackage{amsmath,amssymb,amsthm,mathtools}
\usepackage{graphicx,xcolor,booktabs,tabularx,array}
\usepackage{float}
\usepackage{multirow}
\usepackage{wrapfig}
\usepackage{needspace}
\usepackage{microtype}
\usepackage[section]{placeins}
\usepackage{flafter}
\usepackage{caption}
\usepackage{hyperref}
\usepackage{url}

\iclrfinalcopy
\title{CAST: Reconstruction-Coupled Acceleration\\
of Interactive World Models}
\author{
  Leyang Chen\thanks{Equal contribution},\enspace
  Junyi Wu\footnotemark[1],\enspace
  Fanqing Kong\footnotemark[1],\enspace
  Shaoqiu Zhang,\enspace
  Yulun Zhang\thanks{Corresponding author: Yulun Zhang, yulun100@gmail.com}\\
  Shanghai Jiao Tong University\\
}
\hypersetup{pdftitle={CAST: Reconstruction-Coupled Acceleration of Interactive World Models},pdfauthor={Leyang Chen, Junyi Wu, Shaoqiu Zhang, Fanqing Kong, Yulun Zhang},colorlinks=true,linkcolor=blue,citecolor=blue,urlcolor=blue}
\newcommand{\method}{\textsc{CAST}}
\newcommand{\fc}{PAR}
\newcommand{\norm}[1]{\left\lVert#1\right\rVert}
\newcommand{\conj}[1]{\overline{#1}}
\DeclareMathOperator{\TopK}{TopK}
\DeclareMathOperator{\RFFT}{RFFT2}

\DeclareMathOperator{\Arg}{Arg}
\DeclareMathOperator{\RMS}{RMS}
\DeclareMathOperator{\wrap}{wrap}
\newtheorem{proposition}{Proposition}
\floatstyle{ruled}
\newfloat{algorithm}{tbp}{loa}
\floatname{algorithm}{Algorithm}
\newcolumntype{Y}{>{\raggedright\arraybackslash}X}

\begin{document}
\fussy
\emergencystretch=1em
\parfillskip=0pt plus 1fil
\setlength{\abovedisplayskip}{4pt plus 1pt minus 1pt}
\setlength{\belowdisplayskip}{4pt plus 1pt minus 1pt}
\flushbottom
\clubpenalty=1000
\widowpenalty=1000
\displaywidowpenalty=1000
\maketitle
\lhead{Preprint}
\vspace{-12pt}
\noindent\begin{minipage}{\linewidth}
\centering
\includegraphics[width=0.92\linewidth]{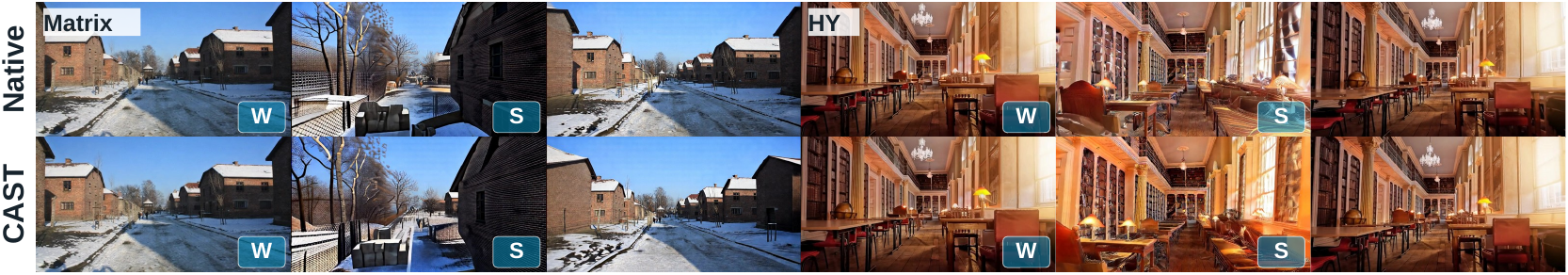}
\captionsetup{type=figure,hypcap=false}
\caption{\textbf{Action-conditioned rollouts.} Native (top) and our method CAST (bottom) on Matrix-Game 3.0 (left) and HY-World 1.5 (right). CAST maintains visual fidelity close to Native while achieving $2.15\times$ and $3.48\times$ speedups, respectively (Figure~\ref{fig:motivation_bc}); W/S denote forward/backward movement.}
\label{fig:motivation_a}
\end{minipage}\par\vspace{-1pt}
\begin{abstract}
\setlength{\leftskip}{-1em}
\setlength{\rightskip}{-1em}
Interactive world models must respond quickly to controls while preserving scene consistency. Existing acceleration methods can miss heterogeneous control responses and spatial transport when recovering skipped features. We observe that interaction-induced feature changes correlate with approximation error, while low-frequency interpolation errors are phase-sensitive and show more predictable phase progression. These findings motivate CAST, a reconstruction-coupled inference framework. CAST selects anchors by interaction sensitivity and cross-layer coverage, reconstructs skipped residuals with frequency- and confidence-aware Phase-Aware Reconstruction (PAR), and coordinates historical KV routing according to downstream reconstruction responsibility. On Matrix-Game 3.0 and HY-World 1.5, CAST achieves $2.15\times$ and $3.48\times$ speedups, respectively, while maintaining visual quality close to Native (Figure~\ref{fig:motivation_a}). It also attains the highest VBench scores among compared methods and leads non-native baselines on seven and six of thirteen WorldMark dimensions, demonstrating a balance of generation speed, visual quality, and interactive responsiveness under real-time control. Code is available at \url{https://github.com/lokiniuniu/CAST}.
\end{abstract}
\vspace{-6pt}
\section{Introduction}
\label{sec:intro}
Interactive world models turn visual generation into a feedback loop: users act and the model generates the next observation~\citep{bruce2024genie,valevski2024gamengen}. Recent systems combine streaming generation, camera and action conditioning, and visual history~\citep{wang2026matrix,sun2025worldplay}. Their quality depends on responsive controls and consistent revisits, beyond plausible frames alone~\citep{xu2026worldmark}. Inference latency remains a practical obstacle to interactive use~\citep{lu2026light}, making efficiency part of the modeling problem.

Efficient sampling algorithms reduce the number of denoising-network evaluations required for generation~\citep{song2020ddim,lu2022dpmsolver}. Distillation and consistency models further enable few-step generation~\citep{salimans2022progressive,song2023consistency}, including autoregressive video synthesis~\citep{yin2025causvid,huang2025selfforcing}. Yet each network evaluation can remain costly even with few denoising steps~\citep{tang2026fis}. CAST targets this remaining cost by reducing computation within each evaluation. Training-free methods exploit redundancy by caching features across denoising steps~\citep{ma2023deepcache,liu2024teacache,lyu2025fastercache}, computing selected frames and interpolating their neighbors~\citep{tang2026fis}, or reducing historical attention~\citep{lu2026light,li2026sol}. These strategies offer complementary savings, although few-step inference limits cross-step reuse~\citep{tang2026fis}. In our setting, the approximations are connected: computed frames supply the endpoints for reconstructing skipped updates, so errors in their historical attention can also propagate beyond anchor frames themselves through reconstruction (Section~\ref{sec:nonlocal}).

\clearpage
\begin{wrapfigure}{r}{0.46\textwidth}
\vspace{-10pt}
\centering
\includegraphics[width=\linewidth]{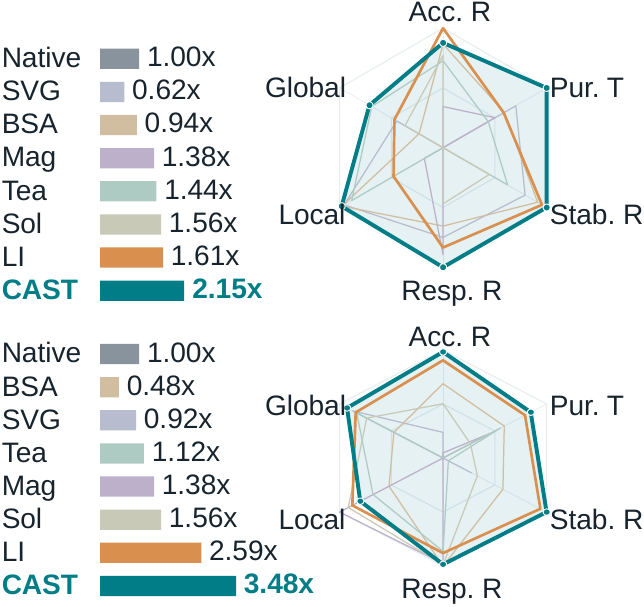}
\captionsetup{type=figure,hypcap=false}
\caption{\textbf{Speed and interactive quality.} Speedup over Native (left) and selected WorldMark dimensions (right), min--max normalized per backbone across all methods including Native; colors match the speedup bars.}
\label{fig:motivation_bc}
\vspace{-10pt}
\end{wrapfigure}

This dependency suggests organizing acceleration around the recovery of skipped states. Prior work shows that motion content and token sensitivity can guide nonuniform computation~\citep{kahatapitiya2024adacache,zou2024toca}. We ask whether control responses provide a useful signal for selecting frame anchors. Figure~\ref{fig:empirical_motivation} shows that stronger camera- and action-induced feature responses are associated with larger next-block skipping errors. Because these responses arise in the conditioning path, they can guide allocation without first evaluating the dense residual we seek to avoid. Selection must also maintain coverage: interleaving frames across layers is important in frame-sparse inference so that exact computation will not be allocated to several fixed frames~\citep{tang2026fis}. We combine sensitivity with a refresh debt to balance responsive updates and regular coverage.

In interactive world models, moving the camera or controlling an object shifts scene content across spatial locations~\citep{wang2026matrix}. Recovering skipped updates between anchors therefore requires accounting for where content moves, beyond how feature values change. Direct interpolation blends endpoint residuals at fixed coordinates, potentially mixing different scene content rather than tracking its displacement. Phase transport offers a way to align these residuals before interpolation: locally coherent motion can be represented through phase shifts~\citep{meyer2015phase}. Our diagnostics in Figure~\ref{fig:spectral_phase_motivation} support this view: oracle phase correction substantially reduces a low-frequency interpolation-error peak, and phase progression is more concentrated at low frequencies. These findings motivate phase-aware reconstruction, with transport strength adapted to frequency and local phase confidence to limit corrections where coherent motion is poorly supported.

Reconstruction changes what an anchor needs from history. Prior analyses show that local approximation errors can have unequal downstream effects~\citep{zou2024toca}. Historical routing in Light Interaction selects blocks by query--key relevance~\citep{lu2026light}. In our setting, an anchor also supplies residuals for skipped frames, so local relevance alone does not account for this dependency. If neighboring anchors omit useful evidence together, their errors may jointly affect the reconstructed interval. We therefore coordinate their historical supports under fixed per-anchor budgets, aiming to preserve complementary evidence for both the anchors and the states they reconstruct.

Building on these observations, we introduce \method{}, a reconstruction-coupled inference framework for interactive world models. \emph{Control-Aware Frame Selection} balances interaction sensitivity with cross-layer coverage. \emph{Phase-Aware Reconstruction} recovers skipped residuals through frequency- and confidence-aware phase transport. Finally, \emph{Reconstruction-Coupled Sparse Attention} refines historical KV supports according to both local omission effects and the recovery of neighboring states. Together, these components connect where computation is retained, how omitted updates are recovered, and which history supports that recovery. Our contributions are as follows:

\begin{itemize}
\setlength{\itemsep}{3pt}
\setlength{\parsep}{0pt}
\vspace{-5pt}
\item \textbf{Control-Aware Frame Selection.} We select anchors using camera- and action-induced feature responses. Cross-layer coverage promotes regular refreshes, balancing interaction-sensitive updates with the need to maintain representations throughout the current chunk and across the entire active action horizon of each generated sequence.

\item \textbf{Phase-Aware Reconstruction.} We recover skipped residuals by aligning endpoint spectra before interpolation or one-sided extrapolation. Transport strength adapts to frequency and local phase confidence, limiting corrections when evidence for coherent motion remains uncertain across the full interactive control sequence.

\item \textbf{Reconstruction-Coupled Sparse Attention.} We coordinate neighboring anchors' historical KV supports under fixed budgets according to their local omission effects and shared responsibility for reconstructing skipped states across adjacent intervals.

\item \textbf{Evaluation on interactive world models.} As shown in Figure~\ref{fig:motivation_bc}, on Matrix-Game 3.0 and HY-World 1.5, \method{} achieves $2.15\times$ and $3.48\times$ acceleration, respectively, and leads non-native methods on seven and six of thirteen WorldMark dimensions.
\end{itemize}

\section{Related Work}
\label{sec:related}\begingroup\setlength{\parskip}{2pt}
\subsection{Interactive Video World Models}

Dreamer and IRIS learn behaviors through imagined rollouts~\citep{hafner2019dreamer,micheli2022iris}. DIAMOND introduces diffusion-based world modeling~\citep{alonso2024diamond}, GameNGen simulates games~\citep{valevski2024gamengen}, and Genie learns interactive environments through latent actions~\citep{bruce2024genie}. Video diffusion and Stable Video Diffusion provide generative foundations~\citep{ho2022videodiffusion,blattmann2023stable}; CameraCtrl and MotionCtrl add explicit camera and motion control~\citep{he2024cameractrl,wang2023motionctrl}, enabling directed changes to viewpoint and scene motion.

\looseness=-1 CausVid and Self Forcing enable few-step autoregressive generation~\citep{yin2025causvid,huang2025selfforcing}. We target interactive chunk-autoregressive models, including Matrix-Game 3.0 and WorldPlay~\citep{wang2026matrix,sun2025worldplay}, where processing frames, controls, and history makes per-evaluation latency critical for responsive online control.

\par\vspace{0pt}
\subsection{Efficient Inference for Video Diffusion Models}
Caching reuses intermediate computation across denoising steps. DeepCache and Block Caching reuse features or block outputs~\citep{ma2023deepcache,wimbauer2023blockcache}; Learning-to-Cache learns layer schedules~\citep{ma2024l2c}. TeaCache and MagCache adapt reuse to input changes, residual magnitudes, or block dynamics~\citep{liu2024teacache,ma2025magcache}. QuantCache combines hierarchical latent and layer caching with importance-guided quantization and pruning for video DiTs~\citep{wu2025quantcache}. AdaCache adapts schedules to video content~\citep{kahatapitiya2024adacache}; TaylorSeer forecasts features instead of directly reusing them~\citep{liu2025taylorseer}.

ToCa selects tokens by sensitivity and redundancy~\citep{zou2024toca}, while DuCa alternates aggressive and conservative caching~\citep{zou2026duca}. These methods establish the value of selective computation, but few-step models offer limited cross-step reuse. There are also other DiT acceleration schemes~\citep{tang2026fis,dcgen,dvd,WorldDynCache,wu2025flashedit}. 

Light Interaction combines context management, output reuse, and historical sparsity~\citep{lu2026light}; Sol Engine integrates caching, sparsity, pruning, quantization, and kernels~\citep{li2026sol}. \method{} coordinates retained frames and historical supports through their shared reconstruction responsibility.

\vspace{0pt}\subsection{Sparse Attention}
Longformer and BigBird use structured sparsity~\citep{beltagy2020longformer,zaheer2020bigbird}; Reformer and Routing Transformer route by content~\citep{kitaev2020reformer,roy2021routing}. FlashAttention and FlashAttention-2 optimize exact attention~\citep{dao2022flashattention,dao2023flash2}. For video, Sliding Tile Attention exploits local windows~\citep{zhang2025sta}, LongCat-Video uses 3D block sparsity~\citep{longcat2025video}, and Sparse VideoGen and VMoBA exploit spatial--temporal structure~\citep{xi2025svg,wu2026vmoba}. Building on pooled historical routing~\citep{lu2026light}, \method{} considers both local omission errors and their propagation to reconstructed frames. This dependency makes retained historical evidence important beyond the computed anchors themselves, linking sparse attention decisions to the accuracy of skipped-frame reconstruction.

\vspace{-4pt}
\endgroup\vspace{0pt}\section{Preliminaries and Motivation}
\label{sec:motivation}
\begin{wrapfigure}[18]{r}{0.5\textwidth}
\vspace{-16pt}\centering\includegraphics[width=\linewidth]{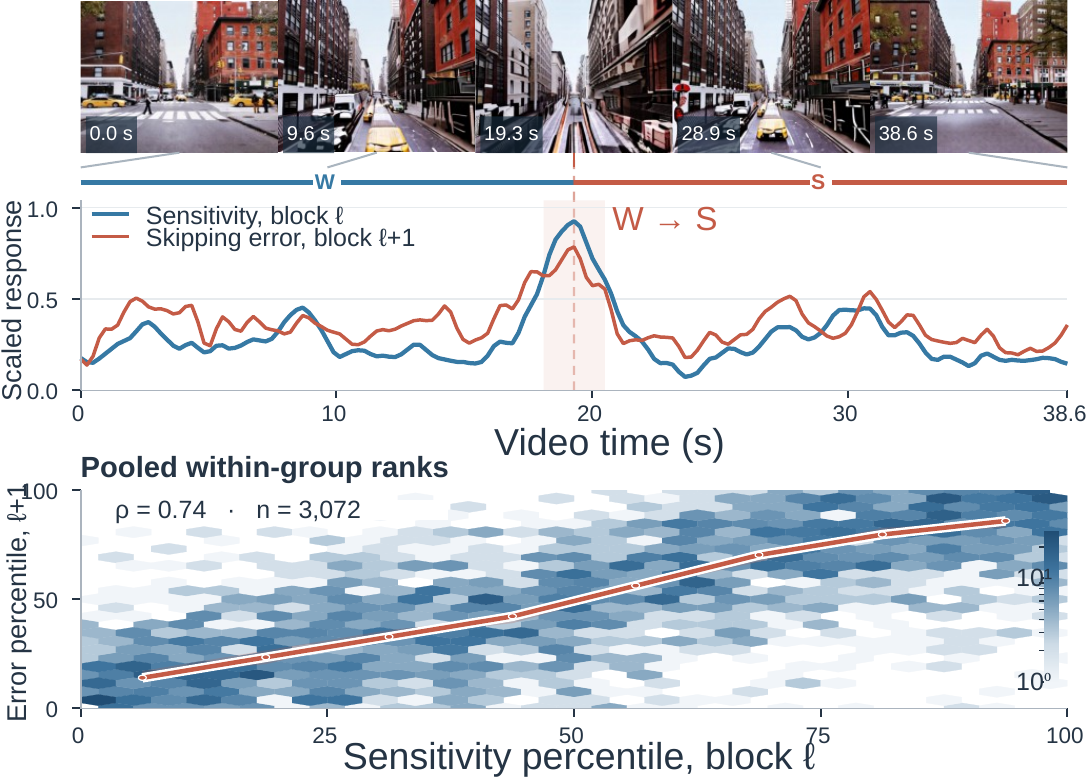}
\caption{\textbf{Control sensitivity and skipping error.} Top: native rollout. Middle: independently scaled sensitivity at block $\ell$ and error at $\ell+1$. Bottom: pooled percentiles (Spearman $\rho=0.74$, $n=3{,}072$), with counts and binned medians.}\label{fig:empirical_motivation}\vspace{-8pt}\end{wrapfigure}

\vspace{-8pt}

Our setting combines diffusion denoising~\citep{ho2020ddpm}, latent representations~\citep{rombach2021ldm}, and transformer blocks~\citep{peebles2022dit}. We study one evaluation of a latent chunk conditioned on controls and visual history. Let $R_i^\ell$ denote an eligible visual residual update for frame $i$ within transformer block $\ell$, and $\widehat R_i^\ell$ its approximation. CAST is applied at eligible visual attention and FFN residual sites, while conditioning and text updates retain native dense execution. Within each network evaluation, retained anchors are computed directly and skipped frames receive reconstructed residual updates. Historical sparsity changes the evidence available to those anchors, linking frame selection, residual recovery, and KV routing within each transformer block.

\par\WFclear\newpage\subsection{Interaction Sensitivity Predicts Frame Approximation Error}
\label{sec:sensitivity}
Let $\mathcal K_\ell\subseteq\{\mathrm{cam},\mathrm{act}\}$ contain the available control branches, with feature updates $\Delta_{i,k}^\ell$. For $\mathcal K_\ell\ne\varnothing$, define control sensitivity and relative approximation error by\par\WFclear
\begin{equation}
d_i^\ell=\frac{1}{|\mathcal K_\ell|}\sum_{k\in\mathcal K_\ell}\operatorname{Norm}(\Delta_{i,k}^\ell),\qquad
e_i^\ell=\frac{\norm{R_i^\ell-\widehat R_i^\ell}_2}{\norm{R_i^\ell}_2+\epsilon}.
\label{eq:sensitivity}
\end{equation}
Here $\operatorname{Norm}$ rescales update magnitudes across the current chunk. Sensitivity is available from the conditioning path; error requires a dense reference and is used only for diagnosis.

Figure~\ref{fig:empirical_motivation} compares $d_i^\ell$ with $e_i^{\ell+1}$, matching the scheduler's one-block prediction lag. Both curves rise near the W-to-S transition; pooled percentiles show positive Spearman correlation ($\rho=0.74$). Because the temporal curves are independently rescaled, their proximity indicates similar temporal variation, not equal magnitudes. The spread in the percentile plot also shows that sensitivity does not fully determine error. This motivates sensitivity-based allocation balanced by cross-layer coverage.

\subsection{Spectral Error and Frequency-Dependent Phase Transport}
\label{sec:spectralmotivation}
\Needspace{235pt}
\begin{wrapfigure}[19]{l}{0.48\textwidth}
\vspace{-2pt}
\centering
\includegraphics[width=\linewidth]{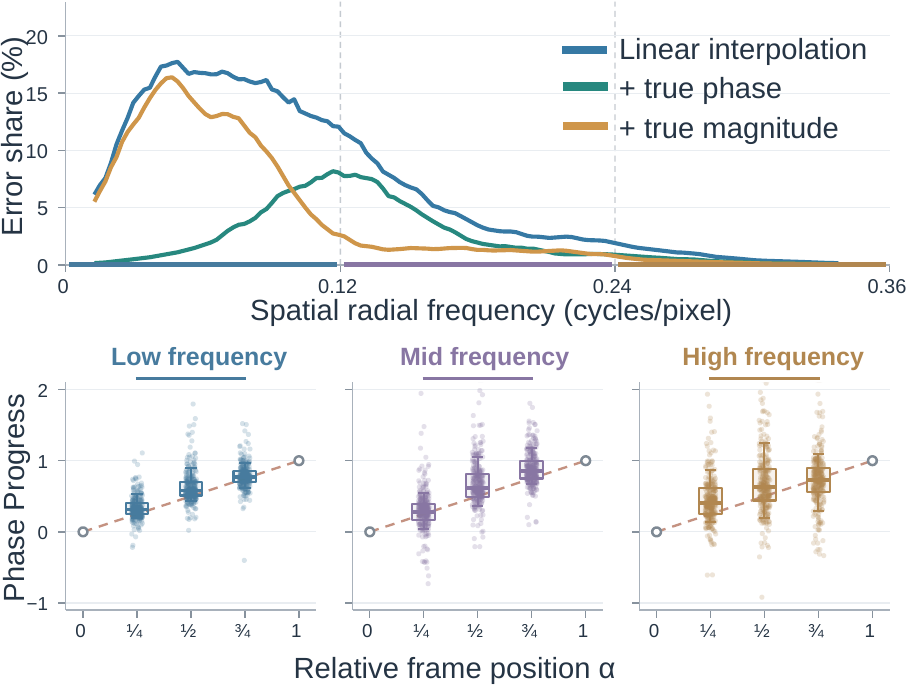}
\caption{\textbf{Spectral error and phase progression.} Top: interpolation error with oracle phase/magnitude corrections. Bottom: phase progress across frequency bands; dashed lines denote uniform progression.}
\label{fig:spectral_phase_motivation}
\end{wrapfigure}
For a skipped frame with $a<t<b$, linear interpolation uses $R_t^{\mathrm{lin}}=(1-\alpha)R_a+\alpha R_b$, where $\alpha=(t-a)/(b-a)$. It mixes fixed coordinates even when content moves. Figure~\ref{fig:spectral_phase_motivation} separates alignment and amplitude errors using oracle corrections that replace either phase or magnitude with its reference value while retaining the other interpolated component (Appendix~\ref{app:diagnostics}).

The upper plot shows a low-frequency error peak largely removed by phase correction, while magnitude correction leaves substantial error. This supports alignment correction in that band. The advantage varies across frequencies, and error share alone does not establish difficulty relative to target-band energy at each frequency.

The lower panels show tighter phase progression at low frequencies and broader distributions at high frequencies. Coherent translation gives linear unwrapped phase at every frequency~\citep{meyer2015phase}; the observed spread reflects reliability, not an intrinsic low-frequency rule. Thus \fc{} prioritizes alignment-sensitive components while tempering transport with local confidence.

\subsection{Reconstruction Dependency Implies Coupled Attention}
\label{sec:nonlocal}
Once two anchors reconstruct an interval, attention approximation at either endpoint can affect every skipped state. For $\widehat R_t=\mathcal P_t(R_a,R_b)$, anchor perturbations induce
\begin{equation}
\Delta_t^{\mathrm{rec}}=\mathcal P_t(R_a+\delta R_a,R_b+\delta R_b)-\mathcal P_t(R_a,R_b).
\label{eq:nonlocal}
\end{equation}
The interval error $\sum_{t\in(a,b)}\norm{\Delta_t^{\mathrm{rec}}}^2$ depends on both endpoints and the reconstruction operator. For linear interpolation, $\Delta_t^{\mathrm{rec}}=(1-\alpha)\delta R_a+\alpha\delta R_b$: the squared error includes a cross term, so endpoint errors can reinforce or offset each other. Choosing historical KV supports independently therefore need not minimize interval error, even when each anchor is locally well approximated. This motivates coordinating supports under fixed per-anchor budgets to protect both anchors and reconstructed states. Section~\ref{sec:mechanisms} tests this dependency at matched frame and historical KV budgets.

\vspace{-5pt}
\par
\section{Method}
\label{sec:method}
\enlargethispage{\baselineskip}

\begin{figure}[t]
\centering
\includegraphics[width=\linewidth]{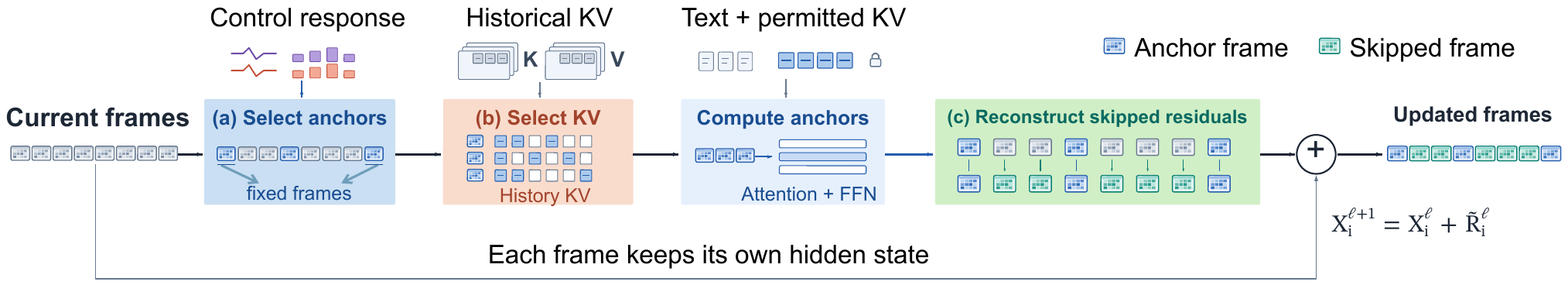}
\caption{\textbf{CAST pipeline.} Select anchors and historical KV blocks, compute anchor residuals, and reconstruct skipped updates while retaining each frame's own skip connection.}
\label{fig:overview}
\end{figure}
\vspace{-3pt}
\subsection{Overview and Problem Formulation}
\label{sec:overview}
\vspace{-3pt}
Write the block input as $X^\ell=[X^{\mathrm{mem},\ell},X^{\mathrm{cur},\ell}]$, with $N_{\mathrm{cur}}$ current latent frames. Memory frames retain native computation or caching; only current residuals are skipped. At block $\ell$, the anchor set $\mathcal A_\ell=\{a_0<\cdots<a_J\}$ contains $K_f=J+1$ frames, and adjacent anchors define intervals $\mathcal I_{a,b}=\{t:a<t<b\}$. The anchor set is selected once per transformer block and shared by its eligible visual attention and FFN residual sites. Suppressing the residual-site index, each eligible update is
\vspace{-7pt}
\begin{equation}
X_i^{\ell+1}=X_i^\ell+\widetilde R_i^\ell,\qquad
\widetilde R_i^\ell=\begin{cases}
R_i^{\ell,\mathrm{sp}},&i\in\mathcal A_\ell,\\
\mathcal P_i(R_a^{\ell,\mathrm{sp}},R_b^{\ell,\mathrm{sp}}),&i\in\mathcal I_{a,b},\\
\mathcal P_i^{\mathrm{ext}}(R_{a_0}^{\ell,\mathrm{sp}},R_{a_1}^{\ell,\mathrm{sp}}),&i<a_0,\\
\mathcal P_i^{\mathrm{ext}}(R_{a_{J-1}}^{\ell,\mathrm{sp}},R_{a_J}^{\ell,\mathrm{sp}}),&i>a_J,
\end{cases}
\label{eq:block}
\end{equation}
Here $\mathcal P_i^{\mathrm{ext}}$ uses the same spectral transport as $\mathcal P_i$ with $\alpha<0$ on the left and $\alpha>1$ on the right. The two anchors nearest each boundary are used: $(a_0,a_1)$ for $i<a_0$ and $(a_{J-1},a_J)$ for $i>a_J$. This construction requires $K_f\geq2$. $R_i^{\ell,\mathrm{sp}}$ denotes the directly computed anchor update at the eligible residual site. Each skipped frame retains its own input $X_i^\ell$; reconstruction supplies only its residual update. Figure~\ref{fig:overview} summarizes the pipeline. Execution routes masks before evaluating anchors and reconstructing skipped updates; the routing surrogate needs no dense anchor outputs.

\vspace{-3pt}
\subsection{Control-Aware Frame Selection}
\vspace{-3pt}
\label{sec:selection}
\begin{wrapfigure}{r}{0.32\textwidth}
\vspace{-38pt}
\centering
\includegraphics[width=\linewidth]{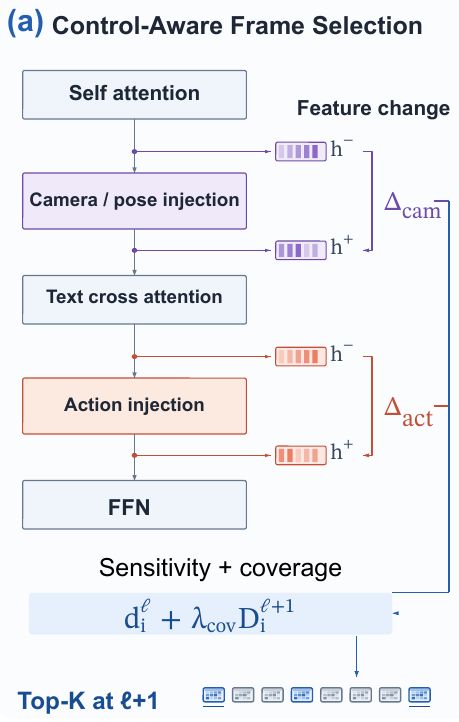}
\caption{Control responses and coverage debt guide selection.}
\label{fig:overview_panel_a}
\vspace{-30pt}
\end{wrapfigure}
\paragraph{Interaction sensitivity.}
The conditioning path supplies $d_i^\ell$ for all current frames (Eq.~\eqref{eq:sensitivity}). We use these scores to select anchors at the next block, so selection does not require the residuals it aims to skip. As illustrated in Figure~\ref{fig:overview_panel_a}, frames with stronger control responses receive higher priority, while coverage debt prevents repeatedly favoring the same subset. Scoring costs are included in the reported end-to-end chunk latency.
\vspace{-8pt}
\paragraph{Coverage debt.}
Following the coverage motivation of frame interleaving~\citep{tang2026fis}, let $h_i^\ell$ count blocks since refresh. Define the refresh age and its normalized coverage debt:\vspace{-4pt}
\begin{equation}
\begin{gathered}
h_i^{\ell+1}=\begin{cases}0,&i\in\mathcal A_\ell,\\h_i^\ell+1,&i\notin\mathcal A_\ell,\end{cases}\\
H_{\mathrm{cov}}=\left\lceil\frac{N_{\mathrm{cur}}}{K_f}\right\rceil,\qquad
D_i^{\ell+1}=\frac{h_i^{\ell+1}}{H_{\mathrm{cov}}}.
\end{gathered}
\label{eq:debt}
\end{equation}
\vspace{-5pt}
\begin{equation}
S_i^{\ell+1}=\begin{cases}
d_i^\ell+\lambda_{\mathrm{cov}}D_i^{\ell+1},&\text{fresh response available},\\
\lambda_{\mathrm{cov}}D_i^{\ell+1},&\text{otherwise}.
\end{cases}
\label{eq:score}
\end{equation}
\enlargethispage{2\baselineskip}
We set $\lambda_{\mathrm{cov}}=0.50$; $H_{\mathrm{cov}}$ scales age but does not guarantee refresh. Deterministic Top-$K_f$ selects $K_f$ anchors across all $N_{\mathrm{cur}}$ current frames: $\mathcal A_{\ell+1}=\TopK_i(S_i^{\ell+1},K_f)$. Neither endpoint is fixed. Without a fresh conditioning response, we omit sensitivity rather than reuse a stale score. Each denoising evaluation starts with a dense block and resets ages (Appendix~\ref{app:selection}).
\vspace{-6pt}
\par\WFclear
\vspace{-3pt}
\subsection{Phase-Aware Reconstruction}
\label{sec:fcpasm}
\vspace{-6pt}
\paragraph{Local phase transport.}
Following phase-based interpolation and extrapolation~\citep{meyer2015phase}, partition residuals into tiles $p$ and suppress the block index. For anchors $a<b$ and target frame $t\notin\{a,b\}$, let $\alpha=(t-a)/(b-a)$ and $F_{i,p,c}(\omega)=\RFFT(R_{i,p,c})(\omega)$ for channel $c$. Interior targets have $0<\alpha<1$; boundary targets have $\alpha<0$ or $\alpha>1$. Figure~\ref{fig:overview_panel_c} illustrates the phase transport and spectral blending used to reconstruct these skipped residuals.

\vspace{-6pt}
Define the pooled cross-spectrum and magnitude product:
\begin{equation}
Z_p(\omega)=\sum_c F_{b,p,c}(\omega)\conj{F_{a,p,c}(\omega)},\quad
W_p(\omega)=\sum_c |F_{b,p,c}(\omega)|\,|F_{a,p,c}(\omega)|.
\label{eq:cross}
\end{equation}
The pooled cross-spectrum favors shared displacement over weak-channel evidence, but it neither estimates optical flow nor resolves phase unwrapping.
\Needspace{0.58\textheight}
\begin{wrapfigure}{r}{0.48\textwidth}
\centering
\includegraphics[width=\linewidth]{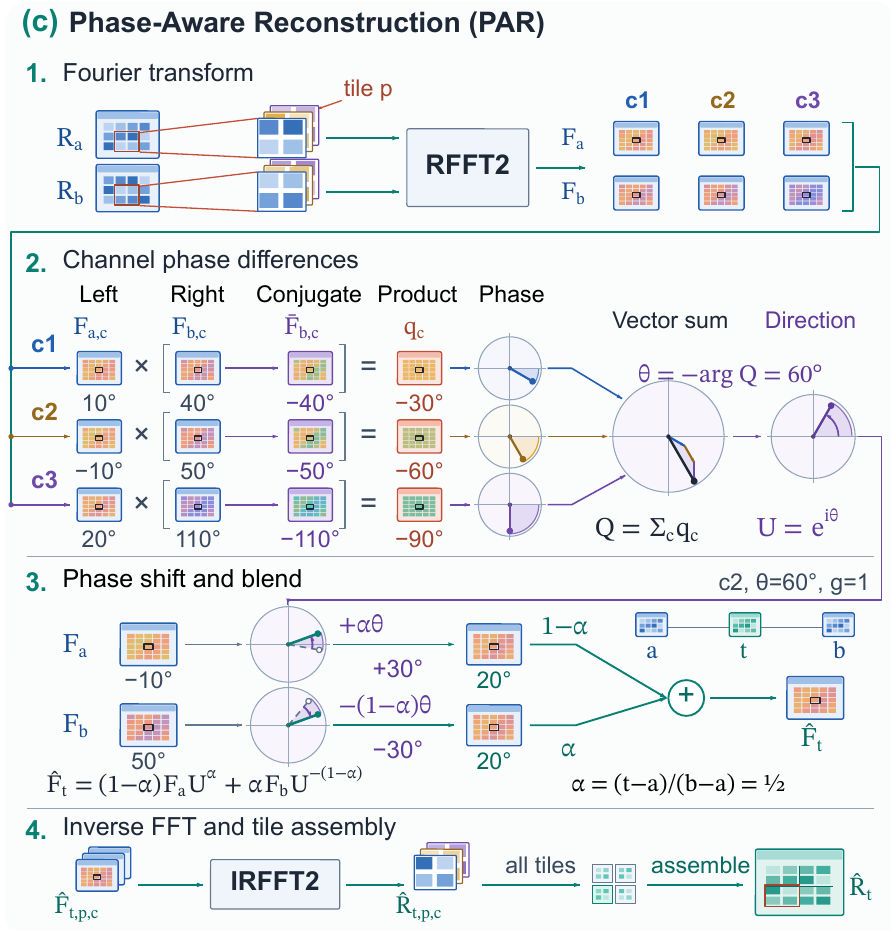}
\captionsetup{skip=4pt}
\vspace{-8pt}
\caption{Phase-aware reconstruction with $Q=\overline{Z}$, where $Z$ is defined in Eq.~\eqref{eq:cross}.}
\label{fig:overview_panel_c}
\vspace{-38pt}
\end{wrapfigure}

\vspace{-3pt}
On reliable bins, set $\theta_p=\Arg Z_p$ and $U_p=e^{\mathrm{i}\theta_p}$; otherwise $U_p=1$. Fractional powers use the fixed branch $U_p^\beta=e^{\mathrm{i}\beta\theta_p}$.

\vspace{-6pt}\paragraph{Frequency-confidence affinity.}
Let $\Omega_+$ be the reliable non-DC spectral bins, with multiplicity weights $\kappa_\omega$ for the real FFT. A tile-level consensus score across the reliable spectral bins is\vspace{-4pt}
\begin{equation}
c_p=\frac{\sum_{\omega\in\Omega_+}\kappa_\omega|Z_p(\omega)|}
{\sum_{\omega\in\Omega_+}\kappa_\omega W_p(\omega)+\epsilon}\in[0,1].
\label{eq:confidence}
\end{equation}
Agreement across channel phases raises $c_p$, whereas conflicting evidence reduces it. With normalized radial frequency $\nu_\omega\in[0,1]$, define
\begin{equation}
\begin{aligned}
\tau(\nu)&=\tau_{\mathrm{low}}+(\tau_{\mathrm{high}}-\tau_{\mathrm{low}})\nu^q,\\
g_{p,\omega}&=\sigma\!\left(\frac{c_p-\tau(\nu_\omega)}{T}\right),
\end{aligned}
\label{eq:gate}
\end{equation}
where $q,T>0$ and $\tau_{\mathrm{high}}\geq\tau_{\mathrm{low}}$. Thus high frequencies require stronger transport evidence. Frequency supplies a prior and confidence estimates reliability, without target access.

\paragraph{Unified reconstruction.}
\fc{} aligns endpoint spectra before interpolation or one-sided extrapolation, applying the following rule to each channel and spatial tile:\vspace{-4pt}
\begin{equation}
\widehat F_{t,p,c}(\omega)=(1-\alpha)F_{a,p,c}(\omega)U_p(\omega)^{g_{p,\omega}\alpha}
+\alpha F_{b,p,c}(\omega)U_p(\omega)^{-g_{p,\omega}(1-\alpha)}.
\label{eq:fcpasm}
\end{equation}
\looseness=-1 An inverse real FFT recovers $\widehat R_{t,p,c}$, and the tiles are assembled into $\widehat R_t=\mathcal P_t(R_a,R_b)$. Setting $g=0$ gives direct value interpolation or extrapolation; setting $g=1$ gives full endpoint phase transport. The anchor endpoints are reproduced exactly. Numerical conventions are given in Appendix~\ref{app:spectral}. 

\par\WFclear
\begin{table}[!t]
\caption{WorldMark results; higher is better. Bold marks the best non-native score per backbone. Speedups are taken from our matched runtime evaluation in Table~\ref{tab:lightinteraction} and are not measured on the WorldMark trajectories used to compute the quality and control scores.}
\label{tab:worldmark}
\centering\scriptsize
\setlength{\tabcolsep}{1.7pt}
\renewcommand{\arraystretch}{1.03}
\resizebox{\linewidth}{!}{%
\begin{tabular}{@{}l*{14}{r}@{}}
\toprule
& \multicolumn{8}{c}{Action Dynamics$\uparrow$}
& \multicolumn{3}{c}{World Memory$\uparrow$}
& \multicolumn{2}{c}{Visual Quality$\uparrow$}
& \multicolumn{1}{c}{Efficiency}\\
\cmidrule(lr){2-9}\cmidrule(lr){10-12}\cmidrule(lr){13-14}\cmidrule(l){15-15}
Method
& \multicolumn{2}{c}{\shortstack{Direction\\Accuracy}}
& \multicolumn{2}{c}{\shortstack{Direction\\Purity}}
& \multicolumn{2}{c}{\shortstack{Motion\\Stability}}
& \multicolumn{2}{c}{\shortstack{Response\\Latency}}
& \multicolumn{1}{c}{Local} & \multicolumn{1}{c}{Global} & \multicolumn{1}{c}{Revisit}
& \multicolumn{1}{c}{Perceptual} & \multicolumn{1}{c}{Aesthetic}
& \multicolumn{1}{c}{Speedup$\uparrow$}\\
\cmidrule(lr){2-3}\cmidrule(lr){4-5}\cmidrule(lr){6-7}\cmidrule(lr){8-9}
& Trans. & Rot. & Trans. & Rot. & Trans. & Rot. & Trans. & Rot.
& & & & & &\\
\midrule
\multicolumn{15}{@{}l}{\textbf{Matrix-Game 3.0}}\\
\addlinespace[1pt]
Native & 98.053 & 87.203 & 83.605 & 82.966 & 79.832 & 98.184 & 93.636 & 96.296 & 76.763 & 47.755 & 80.124 & 77.174 & 58.688 & $1\times$\\
SVG & 97.443 & 82.404 & 84.411 & 77.974 & 80.687 & 97.674 & 92.508 & 96.287 & 76.253 & 45.632 & 78.826 & 75.809 & 58.118 & $0.62\times$\\
BSA & 97.306 & 86.877 & 84.061 & 82.903 & \textbf{82.822} & 98.427 & \textbf{93.748} & 96.268 & 76.502 & 44.809 & 80.591 & \textbf{77.362} & \textbf{59.192} & $0.94\times$\\
MagCache & 97.489 & 84.181 & 83.779 & 79.269 & 72.931 & 92.715 & 92.408 & 96.316 & 67.411 & 43.916 & 77.593 & 71.070 & 53.661 & $1.38\times$\\
TeaCache & 97.972 & 86.158 & 83.572 & 81.514 & 79.612 & 96.616 & 92.949 & 96.137 & 75.433 & 46.561 & 78.638 & 75.654 & 57.635 & $1.44\times$\\
Sol Engine & 97.915 & 86.390 & 82.195 & 81.299 & 73.957 & 95.513 & 91.914 & 96.231 & 65.351 & 45.327 & 77.215 & 72.049 & 55.727 & $1.56\times$\\
Light Interaction & 97.944 & \textbf{87.566} & 84.042 & \textbf{83.392} & 76.010 & 98.695 & 93.728 & 96.304 & 70.812 & 45.720 & 80.999 & 75.302 & 57.653 & $1.61\times$\\
\textbf{\method{}} & \textbf{98.029} & 86.932 & \textbf{85.350} & 81.995 & 79.830 & \textbf{98.973} & 91.523 & \textbf{96.337} & \textbf{76.531} & \textbf{46.645} & \textbf{81.514} & 71.871 & 56.727 & $\mathbf{2.15\times}$\\
\midrule
\multicolumn{15}{@{}l}{\textbf{HY-World 1.5 / WorldPlay}}\\
\addlinespace[1pt]
Native & 90.329 & 87.029 & 81.871 & 84.705 & 56.307 & 85.935 & 83.295 & 83.515 & 91.070 & 55.840 & 87.545 & 74.990 & 58.725 & $1\times$\\
BSA & 91.751 & 82.674 & 78.691 & 80.189 & 52.570 & 88.723 & 85.132 & \textbf{99.118} & 92.510 & 52.945 & 87.884 & \textbf{78.927} & 59.356 & $0.48\times$\\
SVG & 89.773 & 76.378 & 74.093 & 79.965 & 34.723 & 84.741 & 80.642 & 81.549 & 88.634 & 55.142 & 85.737 & 73.290 & 56.371 & $0.92\times$\\
TeaCache & 91.862 & 73.061 & 78.389 & 73.369 & 30.601 & 81.831 & 81.212 & 96.732 & 93.656 & 54.934 & 88.935 & 78.801 & \textbf{59.494} & $1.12\times$\\
MagCache & 89.007 & 73.724 & 77.737 & 74.150 & 32.310 & 81.168 & 80.338 & 98.721 & \textbf{96.053} & 50.327 & 91.135 & 74.887 & 56.768 & $1.38\times$\\
Sol Engine & 91.202 & 80.073 & 76.152 & 78.050 & 29.890 & 85.492 & 73.665 & 98.923 & 95.430 & 54.391 & 90.894 & 77.718 & 58.917 & $1.56\times$\\
Light Interaction & \textbf{92.713} & 85.688 & 80.255 & 83.497 & \textbf{55.552} & 93.450 & 85.060 & 97.003 & 95.136 & 54.979 & \textbf{91.517} & 78.244 & 57.940 & $2.59\times$\\
\textbf{\method{}} & 92.700 & \textbf{86.774} & \textbf{80.686} & \textbf{84.399} & 53.843 & \textbf{94.233} & \textbf{85.930} & 98.838 & 94.566 & \textbf{55.439} & 89.652 & 77.685 & 58.706 & $\mathbf{3.48\times}$\\
\bottomrule
\end{tabular}%
}
\end{table}
\vspace{-6pt}\subsection{Reconstruction-Coupled Sparse Attention}
\label{sec:routing}
\begin{figure}[!htbp]
\centering
\includegraphics[width=\linewidth]{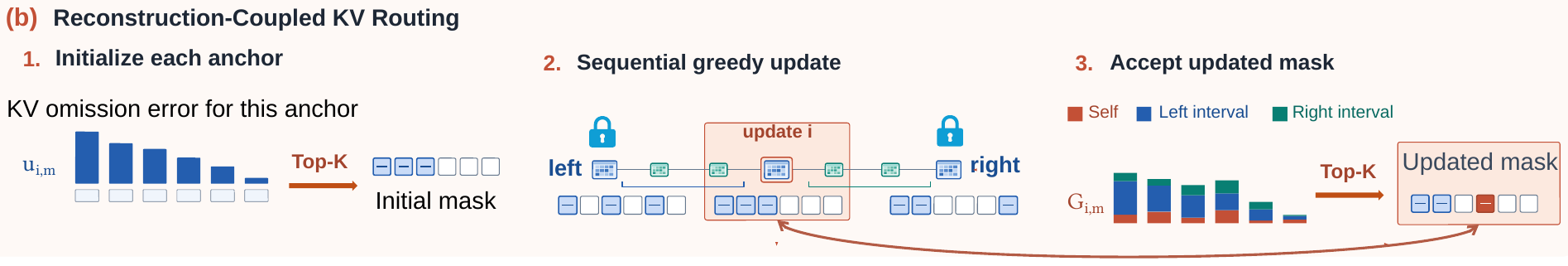}
\caption{Initialize each anchor's mask, then refine it using local and neighboring reconstruction gains while preserving a fixed historical KV budget at every anchor.}
\label{fig:overview_panel_b}
\end{figure}

\paragraph{Local KV-omission importance.}
As in pooled historical routing~\citep{lu2026light}, divide each head's historical KV into $M$ equal-size blocks. Pooled queries and keys/values give
\begin{equation}
s_{i,m}=\frac{\bar q_i^\top\bar k_m}{\sqrt{d_h}},\qquad
\pi_{i,m}=\operatorname{softmax}_m(s_{i,m}),\qquad
\bar o_i=\sum_m\pi_{i,m}\bar v_m.
\label{eq:pooled}
\end{equation}
Pooled one-block omission scores initialize the masks in Figure~\ref{fig:overview_panel_b}:
\begin{equation}
\delta_{i,m}=\frac{\pi_{i,m}}{1-\pi_{i,m}+\epsilon}(\bar o_i-\bar v_m),\qquad
u_{i,m}=\norm{P\delta_{i,m}}_2^2.
\label{eq:omission}
\end{equation}
We use $P=W_O^{(h)}$, the corresponding head slice of the native attention output projection. This inexpensive proxy excludes current/text interactions and later nonlinearities, so it estimates pooled-output deviation rather than the complete nonlinear block error after all residual updates.

\paragraph{Pairwise reconstruction risk.}
Let $x_{i,m}\in\{0,1\}$ indicate whether anchor $i$ retains block $m$. For interval $(a,b)$, define
\begin{equation}
A_{a,b}=\sum_{t\in\mathcal I_{a,b}}(1-\alpha_t)^2,\quad
B_{a,b}=\sum_{t\in\mathcal I_{a,b}}\alpha_t^2,\quad
C_{a,b}=\sum_{t\in\mathcal I_{a,b}}\alpha_t(1-\alpha_t).
\label{eq:abc}
\end{equation}
If phase and confidence are held fixed, the phase factors in Eq.~\eqref{eq:fcpasm} have unit magnitude. The triangle inequality therefore bounds the propagated perturbation by $(1-\alpha_t)\norm{\delta R_a}+\alpha_t\norm{\delta R_b}$. Squaring and summing, then substituting the pooled omission proxies, yields the routing surrogate
\begin{equation}
\begin{aligned}
\mathcal R_{a,b,m}(x_{a,m},x_{b,m})={}&A_{a,b}(1-x_{a,m})u_{a,m}
+B_{a,b}(1-x_{b,m})u_{b,m}\\
&+2C_{a,b}(1-x_{a,m})(1-x_{b,m})\sqrt{u_{a,m}u_{b,m}}.
\end{aligned}
\label{eq:risk}
\end{equation}
The endpoint terms account for each anchor's reconstruction responsibility, while the joint-omission cross term explicitly couples neighboring masks by penalizing simultaneous omission of the same historical block. This frozen-transport surrogate omits cross-block interactions and changes in estimated phase and confidence, so it does not bound the full nonlinear reconstruction error (Appendix~\ref{app:bound}).

\paragraph{Fixed-budget best responses.}
Let $\mathcal E$ contain adjacent-anchor pairs. We minimize
\begin{equation}
\begin{aligned}
\mathcal J(x)={}&\lambda_{\mathrm{anc}}\sum_{i,m}(1-x_{i,m})u_{i,m}
+\lambda_{\mathrm{rec}}\!\left[\sum_{(a,b)\in\mathcal E}\sum_m\mathcal R_{a,b,m}+\mathcal R_{\mathrm{bdry}}(x)\right],\\
&\text{subject to}\quad\sum_m x_{i,m}=K_{\mathrm{KV}}\quad\text{for every anchor }i,
\end{aligned}
\label{eq:objective}
\end{equation}
with nonnegative weights and a fixed $K_{\mathrm{KV}}$. Here $\mathcal R_{\mathrm{bdry}}$ accounts for both anchors used to extrapolate skipped frames outside the anchor span (Appendix~\ref{app:conditional}). Text and permitted current-context interactions remain available; only historical visual blocks enter this budget. We initialize $\mathcal M_i=\TopK_m(u_{i,m},K_{\mathrm{KV}})$. With neighboring masks fixed, define $G_{i,m}$ as the reduction in $\mathcal J$ when anchor $i$ retains block $m$. The conditional objective is linear in the anchor's mask, so its response is
\begin{equation}
\mathcal M_i\leftarrow\TopK_m(G_{i,m},K_{\mathrm{KV}}).
\label{eq:topk}
\end{equation}
Directional refinement passes update each anchor's mask using local gains and reconstruction gains from updated neighboring masks. With strict improvements and ties preserved, repeated passes converge to a coordinate-wise local optimum; the fixed $L_s=2$ passes (forward and backward) need not reach that optimum for every evaluated input.

\section{Experiments}
\label{sec:experiments}
\subsection{Experimental Setup}
\label{sec:setup}\begin{figure}[!t]\centering\includegraphics[width=1.0\linewidth]{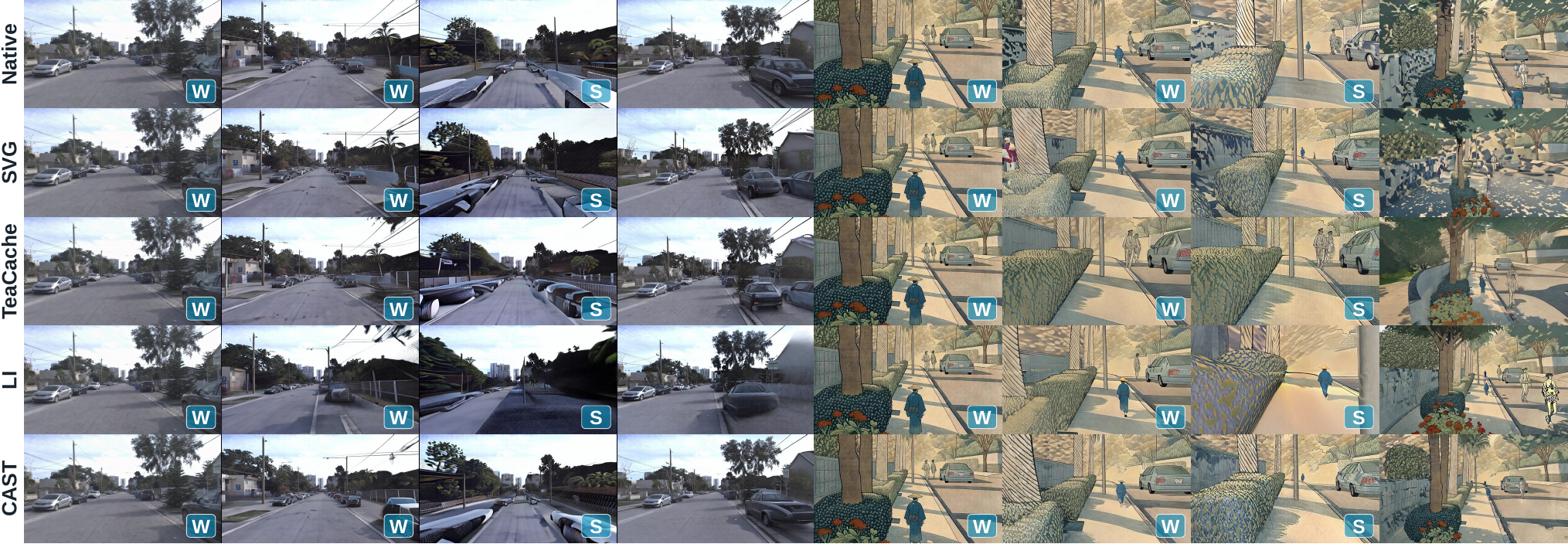}\caption{\textbf{Baseline rollouts on WorldMark.} Matrix-Game 3.0 (left) and HY-World 1.5 (right).}\label{fig:qualitative_baselines}\vspace{-6pt}\end{figure}
\paragraph{Models and protocol.}
We evaluate distilled Matrix-Game 3.0 and HY-World 1.5 (WorldPlay)~\citep{wang2026matrix,sun2025worldplay}, using one and two RTX A6000 GPUs per rollout, respectively; HY uses model/tensor parallel inference. Defaults retain 50\% of current frames and 20\% of historical KV blocks, with two directional refinement passes. Paired comparisons share checkpoints, initial observations, controls, and seeds. Mean steady-state latency uses CUDA synchronization, excludes startup and warm-up, and includes scoring, routing, reconstruction, and memory retrieval; VAE decoding is excluded from all reported steady-state speedup measurements.

\vspace{-2pt}
\paragraph{Benchmarks and baselines.}
WorldMark~\citep{xu2026worldmark} measures thirteen dimensions of control, memory, and visual quality, using 500 samples for every method on both backbones. The Light Interaction-style evaluation~\citep{lu2026light} uses the same 400 trajectories, evaluator, per-backbone hardware, and timing scope for all methods. It reports native-reference and self-comparison PSNR, SSIM~\citep{wang2004ssim}, LPIPS~\citep{zhang2018lpips}, VBench~\citep{huang2023vbench}, and speedup. Native-reference fidelity measures agreement with the base model, not ground-truth correctness. Baselines are Native, SVG~\citep{xi2025svg}, LongCat-Video BSA~\citep{longcat2025video}, TeaCache~\citep{liu2024teacache}, MagCache~\citep{ma2025magcache}, Sol Engine~\citep{li2026sol}, and Light Interaction~\citep{lu2026light}. All baselines are re-run in our unified evaluation environment using their released implementations when available; Tables~\ref{tab:worldmark} and~\ref{tab:lightinteraction} contain no imported published scores. Appendix~\ref{app:config} details the protocols and matched evaluation conditions.

\vspace{-3pt}
\subsection{Main Results: Quality--Efficiency Trade-off}
\label{sec:mainresults}
Tables~\ref{tab:worldmark} and~\ref{tab:lightinteraction} compare behavior and rollout quality, ranking non-native methods.

\Needspace{170pt}
\begin{wraptable}[19]{l}{0.58\textwidth}
\vspace{-2pt}
\caption{VBench and other quality results. Bold: best non-native score per backbone; arrows: preferred direction.}
\label{tab:lightinteraction}
\centering\scriptsize
\setlength{\tabcolsep}{1.2pt}
\renewcommand{\arraystretch}{0.96}
\resizebox{\linewidth}{!}{%
\begin{tabular}{@{}l*{8}{r}@{}}
\toprule
\multirow{2}{*}{Method} & \multicolumn{3}{c}{vs.\ Original} & \multicolumn{3}{c}{Self-Comparison}
& \multirow{2}{*}{VBench$\uparrow$} & \multirow{2}{*}{Speedup$\uparrow$}\\
\cmidrule(lr){2-4}\cmidrule(lr){5-7}
& PSNR$\uparrow$ & SSIM$\uparrow$ & LPIPS$\downarrow$
& PSNR$\uparrow$ & SSIM$\uparrow$ & LPIPS$\downarrow$ & &\\
\midrule
\multicolumn{9}{@{}l}{\textbf{Matrix-Game 3.0}}\\
\addlinespace[1pt]
Native & --- & --- & --- & 17.07 & 0.5264 & 0.3283 & 0.7504 & $1\times$\\
SVG & 12.38 & 0.4170 & 0.5587 & 14.48 & 0.4949 & 0.4406 & 0.7511 & $0.62\times$\\
BSA & 13.34 & 0.4228 & 0.5795 & 16.66 & 0.5326 & 0.4094 & 0.7336 & $0.94\times$\\
MagCache & 15.02 & 0.4902 & 0.4504 & 14.06 & 0.5398 & 0.3837 & 0.7192 & $1.38\times$\\
TeaCache & \textbf{19.03} & \textbf{0.5619} & \textbf{0.3818} & 18.84 & 0.5765 & 0.3602 & 0.7146 & $1.44\times$\\
Sol Engine & 13.81 & 0.4628 & 0.5262 & 16.01 & \textbf{0.6018} & 0.3751 & 0.7151 & $1.56\times$\\
Light Int. & 14.67 & 0.4638 & 0.4442 & 15.59 & 0.4881 & 0.3658 & 0.7415 & $1.61\times$\\
\textbf{\method{}} & 15.73 & 0.4875 & 0.4657 & \textbf{19.35} & 0.5478 & \textbf{0.3581} & \textbf{0.7570} & $\mathbf{2.15\times}$\\
\midrule
\multicolumn{9}{@{}l}{\textbf{HY-World 1.5 / WorldPlay}}\\
\addlinespace[1pt]
Native & --- & --- & --- & 18.60 & 0.5678 & 0.5131 & 0.7201 & $1\times$\\
BSA & 15.94 & 0.4639 & 0.3755 & 15.44 & 0.4205 & 0.3720 & 0.7943 & $0.48\times$\\
SVG & 19.48 & 0.6028 & 0.2209 & 17.75 & 0.5299 & 0.2187 & 0.8082 & $0.92\times$\\
TeaCache & 20.90 & 0.6588 & 0.1892 & 18.86 & 0.5743 & 0.2054 & 0.8150 & $1.12\times$\\
MagCache & 18.76 & 0.6684 & 0.2390 & 18.28 & 0.5597 & 0.2247 & 0.8137 & $1.38\times$\\
Sol Engine & 23.75 & 0.6540 & 0.2574 & 16.33 & 0.3829 & 0.2948 & 0.7293 & $1.56\times$\\
Light Int. & \textbf{24.81} & 0.6500 & \textbf{0.1788} & 18.83 & 0.4570 & 0.5070 & 0.7846 & $2.59\times$\\
\textbf{\method{}} & 22.74 & \textbf{0.6745} & 0.2213 & \textbf{18.89} & \textbf{0.5798} & \textbf{0.2052} & \textbf{0.8178} & $\mathbf{3.48\times}$\\
\bottomrule
\end{tabular}%
}
\end{wraptable}
\textbf{Quality--efficiency comparison.}
\method{} achieves the highest speedup/VBench on both backbones: $2.15\times$/0.7570 on Matrix and $3.48\times$/0.8178 on HY, versus Light Interaction's $1.61\times$/0.7415 and $2.59\times$/0.7846. On HY, TeaCache reaches similar VBench (0.8150) at only $1.12\times$ speedup.

\paragraph{Interactive behavior and world memory.}
\method{} leads seven of thirteen WorldMark dimensions on Matrix and six on HY (Table~\ref{tab:worldmark}). Matrix local/global/revisit scores reach 76.531/46.645/81.514 versus Light Interaction's 70.812/45.720/80.999. On HY, rotational stability/translational response latency improve to 94.233/85.930 from 93.450/85.060; translation stability (53.843 versus 55.552) and revisit memory (89.652 versus 91.517) remain lower.
\vspace{-3pt}
\par\WFclear\paragraph{Rollout fidelity and self-comparison.}
On Matrix, \method{} leads non-native self-comparison PSNR (19.35 versus TeaCache's 18.84) and LPIPS (0.3581 versus 0.3602). It leads all three metrics on HY, narrowly exceeding TeaCache: 18.89 versus 18.86 PSNR, 0.5798 versus 0.5743 SSIM, and 0.2052 versus 0.2054 LPIPS. Figure~\ref{fig:qualitative_baselines} shows matched rollouts.

\subsection{Ablation Studies}
\label{sec:mechanisms}
\label{sec:composition}
\begin{wraptable}{r}{0.58\textwidth}
\vspace{\dimexpr-\intextsep-2\baselineskip\relax}
\vspace{-2pt}
\captionsetup{font=small}
\caption{Progressive configurations and fixed-budget ablations on Matrix-Game 3.0. A: uniform/interleaved anchors + linear reconstruction; B: CAFS + linear reconstruction; C: CAFS + PAR. A--C use independent KV masks. Fixed-budget variants A--D use frame retention $r_f=0.50$ and historical KV retention $r_{\mathrm{KV}}=0.20$. T-Motion, Local, Global, and Revisit are WorldMark scores. Speedup is relative to Native; latency is measured in ms/chunk and excludes VAE decoding for every configuration.}
\label{tab:ablation}
\centering\fontsize{7.5}{8.5}\selectfont
\setlength{\tabcolsep}{1pt}
\renewcommand{\arraystretch}{1.12}
\resizebox{\linewidth}{!}{%
\begin{tabular}{@{}lrrrrrrrr@{}}
\toprule
Variant & VBench$\uparrow$ & LPIPS$\downarrow$ & \shortstack{T-Motion$\uparrow$} & Local$\uparrow$ & Global$\uparrow$ & Revisit$\uparrow$ & Latency$\downarrow$ & Speedup$\uparrow$\\
\midrule
A: Uniform & 0.7218 & 0.5128 & 74.218 & 70.864 & 38.927 & 72.991 & 1,934 & $2.23\times$\\
B: + CAFS & 0.7346 & 0.4974 & 76.482 & 72.943 & 41.586 & 75.648 & 1,951 & $2.21\times$\\
C: + PAR & 0.7479 & 0.4776 & 78.219 & 75.046 & 44.387 & 79.206 & 1,976 & $2.19\times$\\
\midrule
Native & 0.7504 & --- & 79.832 & 76.763 & 47.755 & 80.124 & 4,319 & $1.00\times$\\
+ CAFS & 0.7370 & 0.4960 & 76.800 & 74.000 & 43.200 & 76.700 & 2,800 & $1.54\times$\\
\shortstack[l]{+ CAFS+ PAR} & 0.7600 & 0.4550 & 80.200 & 76.900 & 47.500 & 82.000 & 2,825 & $1.53\times$\\
\shortstack[l]{D: Full \method{}} & 0.7570 & 0.4657 & 79.830 & 76.531 & 46.645 & 81.514 & 2,007 & $2.15\times$\\
\bottomrule
\end{tabular}%
}
\vspace{-\intextsep}
\end{wraptable}
Table~\ref{tab:ablation} compares progressive/fixed-budget designs on Matrix-Game 3.0. Native retains all frames and historical KV; intermediate configurations retain 50\% of frames with dense KV, and full \method{} retains 20\% of historical KV. Fixed-budget comparisons match checkpoints, prompts, controls, noise, candidate sets, and budgets.

Full \method{} cuts latency from 4,319 to 2,007 ms/chunk ($2.15\times$), raising VBench from 0.7504 to 0.7570 and revisit memory from 80.124 to 81.514. T-Motion/local memory remain near Native; global memory falls from 47.755 to 46.645. Native + CAFS uses linear residual interpolation and two-anchor linear extrapolation outside the anchor span. PAR replaces recovery at matched frame retention; RCSA adds sparse historical access with coupled masks. RCSA changes KV retention; C to D isolates coupling at fixed sparsity.
\vspace{-6pt}
\paragraph{Contribution at fixed budgets.}
At fixed frame/KV retention, A uses uniform/interleaved anchors, linear reconstruction, and independent KV masks; B adds CAFS, C replaces reconstruction with PAR, and D adds coupling (full \method{}). A to B improves T-Motion/global memory by 2.264/2.659 points. B to C lowers LPIPS by 0.0198 and adds 3.558 revisit points. C to D adds 2.258/2.308 global/revisit points and lowers LPIPS by 0.0119. Respective costs are 17/25/31 ms/chunk. A to D lowers LPIPS by 0.0471 and adds 7.718 global-memory points for 73 ms/chunk. Coupling trades runtime for quality; gains depend on addition order.
\subsection{Efficiency Analysis}
\label{sec:efficiency}
\paragraph{Runtime and overhead.}

\begin{wraptable}{r}{0.48\textwidth}
\vspace{-35pt}
\caption{Mean steady-state runtime (ms/chunk); OH includes routing and reconstruction.}
 \label{tab:runtime}
 \centering\fontsize{7.5}{8.5}\selectfont
 \setlength{\tabcolsep}{2pt}
 \renewcommand{\arraystretch}{0.86}
 \begin{tabular*}{\linewidth}{@{\extracolsep{\fill}}lrrrrrc@{}}
 \toprule
 Variant & Core & Route & Recon. & Total & OH & Speedup\\
 \midrule
 \multicolumn{7}{@{}l}{\textbf{Matrix-Game 3.0}}\\
Native & 4,319 & 0 & 0 & 4,319 & 0.0\% & $1.00\times$\\
 Indep. (C) & 1,887 & 43 & 46 & 1,976 & 4.5\% & $2.19\times$\\
 \method{} & 1,891 & 74 & 42 & \textbf{2,007} & 5.8\% & $\mathbf{2.15\times}$\\
 \midrule
 \multicolumn{7}{@{}l}{\textbf{HY-World 1.5}}\\
Native & 8,713 & 0 & 0 & 8,713 & 0.0\% & $1.00\times$\\
 Indep. (C) & 2,303 & 55 & 74 & 2,432 & 5.3\% & $3.58\times$\\
 \method{} & 2,297 & 127 & 80 & \textbf{2,504} & 8.3\% & $\mathbf{3.48\times}$\\
 \bottomrule
 \end{tabular*}
 \vspace{-10pt}
\end{wraptable}
Table~\ref{tab:runtime} compares Native, independent routing (C), and \method{} with matched precision, resolution, chunk shape, sampling, and hardware per backbone. Steady-state latency includes selection, routing, reconstruction, and retrieval, excluding warm-up and VAE decoding. C combines CAFS/PAR with independent historical KV routing.

\textbf{Overhead.}
Core time falls from 4,319 ms to 1,887/1,891 ms for C/\method{} on Matrix and from 8,713 ms to 2,303/2,297 ms on HY. \method{}'s routing/reconstruction costs 74/42 ms on Matrix and 127/80 ms on HY (5.8\%/8.3\% of latency). Coupling adds 31/72 ms/chunk over C, reducing speedup from $2.19\times$/$3.58\times$ to $2.15\times$/$3.48\times$. On Matrix, this accompanies Table~\ref{tab:ablation}'s memory/LPIPS gains. HY's larger native-to-accelerated core-time ratio yields greater speedup despite its higher overhead share.
\vspace{0pt}
\paragraph{Default operating point.}
Appendix~\ref{app:budget_sweep} identifies $(r_f,r_{\mathrm{KV}})=(0.50,0.20)$ as a quality--latency knee. At 20\% KV retention, cutting frame retention from 50\% to 30\% raises speedup from $2.15\times$ to $2.52\times$, but lowers VBench from 0.7570 to 0.7441 and global memory from 46.645 to 42.738. Raising frame retention to 60\% adds only 0.0031 VBench/0.574 global-memory points for 356 ms/chunk. At 50\% frame retention, 30\% KV retention adds 0.741/0.749 global/revisit points for 71 ms/chunk. Four versus two passes add only 0.136/0.158 global/revisit points for 31 ms/chunk; one forward--backward pair ($L_s=2$) captures most coupling gains at the selected operating point.
\par\WFclear\vspace{-7pt}
\section{Conclusion}
\label{sec:conclusion}\vspace{-10pt}
We introduced \method{}, which connects anchor selection, residual reconstruction, and historical KV routing through their downstream effects. It balances control sensitivity with coverage, recovers skipped residuals using \fc{}, and coordinates historical supports to protect that recovery. \method{} achieves $2.15\times$ and $3.48\times$ speedups on Matrix-Game 3.0 and HY-World 1.5, with VBench scores of 0.7570 and 0.8178. It also leads non-native methods on seven and six of thirteen WorldMark dimensions, respectively, demonstrating benefits for interactive behavior and visual quality. These findings highlight reconstruction dependencies as a useful guide for coordinating frame computation and historical attention under limited inference budgets.

\FloatBarrier
\bibliography{iclr2027_conference}
\bibliographystyle{iclr2027_conference}
\newpage
\appendix
\raggedbottom
\setcounter{topnumber}{3}
\setcounter{totalnumber}{4}

\noindent\textbf{Appendix overview.} The appendix provides frame-scheduling and spectral-reconstruction details, the reconstruction-risk analysis, and the configuration and evaluation protocols supporting the main-text results and their interpretation under the stated modeling assumptions.

\section{Frame Scheduling and Spectral Reconstruction Details}
\label{app:details}

\subsection{Control-Score Normalization and Scheduling}
\label{app:selection}
For module $k\in\{\mathrm{cam},\mathrm{act}\}$, define $r_{i,k}^\ell=\RMS(\Delta_{i,k}^\ell)$ over the frame's spatial and channel coordinates. We use independent min--max normalization across the current chunk:
\begin{equation}
\operatorname{Norm}(\Delta_{i,k}^\ell)=
\frac{r_{i,k}^\ell-\min_j r_{j,k}^\ell}
{\max_j r_{j,k}^\ell-\min_j r_{j,k}^\ell+\epsilon}.
\label{eq:normalization}
\end{equation}
A constant module response contributes zero. If only one control branch exists, average the available branch rather than treating the absent branch as evidence of low sensitivity. This convention bounds $d_i^\ell$ in $[0,1]$ and prevents arbitrary feature scales from dominating the debt term. It is not robust to outliers, so this normalization can be sensitive to extreme responses.

All $N_{\mathrm{cur}}$ current frames participate in deterministic Top-$K_f$ selection under Eq.~\eqref{eq:score}; ties use increasing frame index. Neither the first nor the last frame is mandatory. Matrix selects five of ten current latent frames, and HY selects two of four. When $K_f=N_{\mathrm{cur}}$, all current frames are computed and reconstruction is bypassed. Positive coverage debt raises the priority of repeatedly skipped frames; $H_{\mathrm{cov}}$ remains a nominal age scale rather than a strict refresh guarantee. Skipped frames between two selected anchors use PAR, while frames outside the selected-anchor span use one-sided PAR extrapolation from the two nearest anchors on that side.

The first block of each denoising evaluation is dense and resets ages. An anchor set is then selected once per block and shared across eligible visual attention and FFN residual sites. Fresh all-frame conditioning responses from block $\ell$ supply $d_i^\ell$ for block $\ell+1$; where no fresh response is available, selection uses coverage debt alone. Conditioning and text updates remain dense, so the score does not require computing the visual residual being skipped.

\subsection{Phase, Confidence, and Real-Valued Reconstruction}
\label{app:spectral}
\looseness=-1 Use spatial FFTs independently in each channel, with orthonormal normalization. This paper uses disjoint tiles; smaller boundary tiles are transformed at their own size, avoiding hidden padding amplification in the norm argument. Overlapping-window variants are outside the evaluation scope.

\looseness=-1 The ratio $|Z_p(\omega)|/(W_p(\omega)+\epsilon)$ is an amplitude-weighted circular resultant length across channels. It approaches one when nonzero channel cross-spectra share a phase, and approaches zero when those phases cancel. Equation~\eqref{eq:confidence} aggregates this evidence over reliable non-DC bins. Reliability requires $W_p$ and $|Z_p|$ to exceed documented numerical thresholds; the thresholds are $W_p(\omega)>\eta_p$ and $|Z_p(\omega)|>\eta_p$, where $\eta_p=10^{-6}\max_{\omega}W_p(\omega)+10^{-8}$, with FP32 accumulation and $\epsilon=10^{-8}$. If no reliable bin exists, set $c_p=0$ and $U_p=1$. A confidence score is not a calibrated probability of correct motion or a guarantee that phase transport will recover the target.

For a real transform, mirrored frequency bins must receive conjugate phase factors. Define phases on one representative of each conjugate pair and assign the negative phase to its mirror. Self-conjugate bins, including DC and applicable Nyquist points, receive $U_p=1$ so that their coefficients remain real. The real-FFT weights $\kappa_\omega$ account for the omitted conjugate half when summing energies. With these conventions the full-spectrum multiplier has unit modulus and preserves real-valuedness. Radial frequency is normalized by the largest represented spatial frequency radius.

For a translation with an unambiguous phase branch, substitution into Eq.~\eqref{eq:fcpasm} at $g=1$ gives
\begin{equation}
\widehat F_t=(1-\alpha)F_a e^{-\mathrm{i}\alpha\omega^\top d}
+\alpha F_a e^{-\mathrm{i}\omega^\top d}e^{\mathrm{i}(1-\alpha)\omega^\top d}
=F_a e^{-\mathrm{i}\alpha\omega^\top d}.
\label{eq:exacttranslation}
\end{equation}
This identity also holds for real $\alpha$ outside $[0,1]$ under the same translation and phase-branch assumptions. Boundary reconstruction uses Eq.~\eqref{eq:fcpasm} without clipping $\alpha$ or adding anchors. Extrapolation is not a convex blend: even with fixed transport, its worst-case perturbation amplification is $|1-\alpha|+|\alpha|$. Long extrapolation distances, changing motion, and occlusion can therefore degrade recovery. The routing surrogate accounts for propagated anchor errors, not the intrinsic error of extending a local motion model beyond its supporting anchors.
\looseness=-1 This identity applies to coherent, identifiable transport modes, not arbitrary deformation or ambiguous Nyquist content. Principal endpoint phase only determines displacement modulo $2\pi$ at each frequency. Fractional powers of wrapped phase are generally not equal to fractional powers of the true unwrapped displacement. High consensus can coexist with wrapping ambiguity; the method does not claim to remove this limitation. Phase wrapping remains a limitation of principal-branch transport. The confidence- and frequency-dependent gate reduces aggressive corrections in unreliable regions but does not resolve phase ambiguity.

\section{Reconstruction Risk and Historical Routing}
\label{app:bound}
\subsection{Exact Single-Block Omission in the Pooled Model}
Removing historical block $m$ and renormalizing the remaining pooled attention gives, for $\pi_{i,m}<1$,
\begin{equation}
\bar o_i^{(-m)}=\frac{\bar o_i-\pi_{i,m}\bar v_m}{1-\pi_{i,m}},\qquad
\bar o_i^{(-m)}-\bar o_i=\frac{\pi_{i,m}}{1-\pi_{i,m}}(\bar o_i-\bar v_m).
\label{eq:omissionexact}
\end{equation}
The positive stabilizer in Eq.~\eqref{eq:omission} modifies this identity near unit mass. It is exact only before stabilization and only for the stated pooled distribution. We use $P=W_O^{(h)}$, the corresponding head slice of the native attention output projection, with no additional learned or sketched projection. A pooled projected vector does not upper-bound a spatial residual tensor without further assumptions. The proxy concerns the attention residual site; subsequent FFN nonlinearities are not covered by this local estimate. The FFN residual is computed and reconstructed separately using the same block-level anchors as the corresponding visual attention residuals.

\subsection{A Conditional Propagation Bound}
\label{app:conditional}
Freeze $U_p$ and $g_{p,\omega}$ while perturbing anchor residuals. Let $T_{a,t}$ and $T_{b,t}$ denote the resulting phase multipliers and transforms, including disjoint-tile assembly. Parseval's identity gives $\norm{T_{a,t}v}_2=\norm{v}_2$ and similarly for $T_{b,t}$ under Appendix~\ref{app:spectral}'s conventions.

\begin{proposition}[Frozen-transport interval bound]
For endpoint perturbations $v_a,v_b$ and fixed transport operators, the reconstruction change satisfies
\begin{equation}
\sum_{t\in\mathcal I_{a,b}}\norm{(1-\alpha_t)T_{a,t}v_a+\alpha_tT_{b,t}v_b}_2^2
\leq A_{a,b}\norm{v_a}_2^2+B_{a,b}\norm{v_b}_2^2
+2C_{a,b}\norm{v_a}_2\norm{v_b}_2.
\label{eq:bound}
\end{equation}
\end{proposition}
\begin{proof}
Expand the squared norm at each $t$. The two individual terms follow from norm preservation. The cross term is at most $2\alpha_t(1-\alpha_t)\norm{v_a}_2\norm{v_b}_2$ by Cauchy--Schwarz. Summing over the interval gives the three interval-dependent coefficients used in Eq.~\eqref{eq:abc}.
\end{proof}

For a single omitted block, insert $v_a=(1-x_{a,m})\delta R_{a,m}$ and $v_b=(1-x_{b,m})\delta R_{b,m}$. Binary masks satisfy $(1-x)^2=1-x$. Replacing their true squared residual norms by $u_{a,m}$ and $u_{b,m}$ produces Eq.~\eqref{eq:risk}. For integer frame indices with gap $n=b-a\geq1$, the coefficients have closed forms
\begin{equation}
A_{a,b}=B_{a,b}=\frac{(n-1)(2n-1)}{6n},\qquad
C_{a,b}=\frac{n^2-1}{6n}.
\label{eq:abcclosed}
\end{equation}
They vanish for adjacent frames ($n=1$), which have no skipped-frame responsibility.

For zero-based current-frame indices, define the left and right boundary sets
$\mathcal I_L=\{t:0\leq t<a_0\}$ and $\mathcal I_R=\{t:a_J<t<N_{\mathrm{cur}}\}$. Boundary extrapolation uses $(a_0,a_1)$ on the left and $(a_{J-1},a_J)$ on the right. For $q\in\{L,R\}$, let $(a_q,b_q)$ denote the corresponding pair and define
\begin{equation}
\begin{aligned}
A_q&=\sum_{t\in\mathcal I_q}(1-\alpha_t)^2,\qquad
B_q=\sum_{t\in\mathcal I_q}\alpha_t^2,\\
C_q&=\sum_{t\in\mathcal I_q}|\alpha_t(1-\alpha_t)|,\qquad
\alpha_t=\frac{t-a_q}{b_q-a_q}.
\end{aligned}
\label{eq:boundarycoeff}
\end{equation}
For frozen transport, the per-frame perturbation is bounded by $|1-\alpha_t|\norm{v_{a_q}}_2+|\alpha_t|\norm{v_{b_q}}_2$. Squaring and summing gives the boundary analogue of Eq.~\eqref{eq:bound}; its cross coefficient is $C_q\geq0$, even outside $[0,1]$. Substituting the pooled omission proxies gives
\begin{equation}
\begin{aligned}
\mathcal R_{\mathrm{bdry}}(x)=\sum_{q\in\{L,R\}}\sum_m\big[{}&A_q(1-x_{a_q,m})u_{a_q,m}+B_q(1-x_{b_q,m})u_{b_q,m}\\
&+2C_q(1-x_{a_q,m})(1-x_{b_q,m})\sqrt{u_{a_q,m}u_{b_q,m}}\big].
\end{aligned}
\label{eq:boundaryrisk}
\end{equation}
Empty boundary sets contribute zero. When $K_f=2$, both boundary sets use the same anchor pair, but contain disjoint target frames; both contributions are included. Interior intervals contribute pairwise responsibility through Eq.~\eqref{eq:risk}; boundary pairs contribute the same responsibility under one-sided extrapolation. Together they give Eq.~\eqref{eq:objective}.

Three qualifications matter. First, recomputing $U$ and $g$ after an anchor perturbation introduces an additional operator-change term; the proposition does not bound it. Second, pooling and projection can distort omission norms. Third, summing per-block risks discards interactions among different omitted blocks and their joint softmax renormalization. Equation~\eqref{eq:objective} is therefore a tractable surrogate rather than a global certificate for sparse inference. Moreover, the conservative norm bound removes the phase orientation itself: its coefficients depend on temporal interpolation responsibility, not on the realized phase or confidence values. The interior bound also applies to direct linear interpolation; the boundary version uses absolute cross coefficients for extrapolation. The claim supported by this derivation is coupling to the reconstruction dependency and norm-preserving operator class; phase-specific routing benefits require empirical evidence or a sharper estimator.

\subsection{Retain Gains and Coordinate-Wise Refinement}
\label{app:solver}
For anchor $i$ with left/right neighbors $l$/$r$ when present, fix their masks and define the retain gain
\begin{equation}
\begin{aligned}
G_{i,m}={}&\lambda_{\mathrm{anc}}u_{i,m}\\
&+\lambda_{\mathrm{rec}}\mathbf1_{l\text{ exists}}
\left[B_{l,i}u_{i,m}+2C_{l,i}(1-x_{l,m})\sqrt{u_{l,m}u_{i,m}}\right]\\
&+\lambda_{\mathrm{rec}}\mathbf1_{r\text{ exists}}
\left[A_{i,r}u_{i,m}+2C_{i,r}(1-x_{r,m})\sqrt{u_{i,m}u_{r,m}}\right]\\
&+\lambda_{\mathrm{rec}}G^{\mathrm{bdry}}_{i,m}.
\end{aligned}
\label{eq:gain}
\end{equation}
where the boundary contribution is
\begin{equation}
\begin{aligned}
G^{\mathrm{bdry}}_{i,m}=\sum_{q\in\{L,R\}}\big\{&\mathbf1_{i=a_q}\left[A_qu_{i,m}+2C_q(1-x_{b_q,m})\sqrt{u_{i,m}u_{b_q,m}}\right]\\
+&\mathbf1_{i=b_q}\left[B_qu_{i,m}+2C_q(1-x_{a_q,m})\sqrt{u_{i,m}u_{a_q,m}}\right]\big\}.
\end{aligned}
\label{eq:boundarygain}
\end{equation}
Holding all other masks fixed, $\mathcal J=\mathrm{constant}-\sum_mx_{i,m}G_{i,m}$. Thus the largest $K_{\mathrm{KV}}$ gains solve the single-anchor problem exactly. A neighboring retention removes the joint-omission contribution, showing complementary coupling; large individual gains can still favor shared retention.

For efficiency, an implementation may restrict each anchor to a fixed candidate set $\mathcal C_i$ of size $M_c\geq K_{\mathrm{KV}}$, obtained from its local scores and neighboring proposals. Candidate sets are fixed before refinement; discarded coordinates remain zero and still count as omissions in neighboring gains. The same candidate sets must be used by the independent baseline. The convergence statement is then relative to these sets. Full-history pooled scoring and candidate construction are still counted in runtime; a small $M_c$ does not make candidate discovery free. The unrestricted mathematical objective is recovered with $\mathcal C_i=\{1,\ldots,M\}$, with no candidate truncation.

\begin{algorithm}[htb]
\caption{Fixed-budget reconstruction-coupled routing}
\label{alg:routing}
\small
\begin{tabularx}{\linewidth}{@{}rY@{}}
1 & \textbf{Input:} ordered anchors, omission scores $u$, fixed candidates $\mathcal C_i$, interval coefficients, budget $K_{\mathrm{KV}}$, boundary coefficients from Eq.~\eqref{eq:boundarycoeff}, and directional-pass cap $L_s$.\\
2 & Initialize $\mathcal M_i\gets\TopK_{m\in\mathcal C_i}(u_{i,m},K_{\mathrm{KV}})$ for every anchor; construct binary masks $x$.\\
3 & \textbf{For} directional refinement pass $s=1,\ldots,L_s$:\\
4 & \quad Set $\mathrm{changed}\gets\mathrm{false}$.\\
5 & \quad \textbf{For} $i$ in forward order if $s$ is odd, backward order otherwise:\\
6 & \qquad Compute $G_{i,m}$ using the latest neighbor masks (Eq.~\eqref{eq:gain}).\\
7 & \qquad Let $\mathcal M_i'\gets\TopK_{m\in\mathcal C_i}(G_{i,m},K_{\mathrm{KV}})$.\\
8 & \qquad \textbf{If} $\sum_{m\in\mathcal M_i'}G_{i,m}>\sum_{m\in\mathcal M_i}G_{i,m}$:\\
9 & \qquad\quad Replace $\mathcal M_i$, update $x$, and set $\mathrm{changed}\gets\mathrm{true}$.\\
10 & \quad \textbf{If} no mask changed, stop.\\
11 & \textbf{Return:} historical masks with exactly $K_{\mathrm{KV}}$ retained blocks per anchor.\\
\end{tabularx}
\end{algorithm}

\begin{proposition}[Termination and conditional optimality]
With fixed scores, coefficients, and candidate sets, exact Top-$K$ best responses that accept only strict decreases in $\mathcal J$ terminate after finitely many changes. A full unchanged directional refinement pass certifies coordinate-wise optimality over the candidates while holding all neighboring masks fixed.
\end{proposition}
\begin{proof}
There are at most $\prod_i\binom{|\mathcal C_i|}{K_{\mathrm{KV}}}$ feasible masks. Every accepted change strictly decreases $\mathcal J$, so no mask configuration can recur. At termination, the best response at every anchor fails to improve the objective while all other masks remain fixed. No feasible single-anchor replacement can therefore improve it. This does not imply a globally optimal joint mask.
\end{proof}
Tie handling is necessary: deterministic tie-breaking alone does not establish strict descent. In floating point, accept only a decrease larger than a reported tolerance, yielding an approximate coordinate-wise condition. Stopping at the directional-pass cap preserves the budget but provides no convergence certificate. In the special cases $K_{\mathrm{KV}}=M$ or $M=0$, historical routing is bypassed; if fewer than $K_{\mathrm{KV}}$ blocks exist, all are kept and the actual available budget is reported.

\subsection{Block Execution Order}
\looseness=-1 Select the anchor set once per block using fresh lagged control scores and updated debt, or debt alone when no fresh score is available. At each eligible visual attention site, form pooled historical statistics and refine masks using Algorithm~\ref{alg:routing}; evaluate anchor attention residuals and reconstruct skipped updates. At each eligible FFN site, compute anchor FFN residuals and reconstruct the remaining updates separately with the same anchors. Every update is added to its own current-frame hidden state. Conditioning and text updates remain dense, all permitted current-frame K/V are constructed, and native causal permissions are preserved. Debt and fresh all-frame control signals then determine the next block's anchors using information already available from the current evaluation.

\subsection{Complexity and Implementation}
\label{sec:complexity}
\label{app:complexity}
Let each current frame contain $S$ tokens of width $D$, each historical block contain $B_v$ tokens, and $N_0$ denote retained current/text KV tokens. Direct current-frame attention changes from $O(N_{\mathrm{cur}}S(N_0+MB_v)D)$ to $O(K_fS(N_0+K_{\mathrm{KV}}B_v)D)$, and directly evaluated current-frame FFNs change from $O(N_{\mathrm{cur}}SD^2)$ to $O(K_fSD^2)$. Any all-frame conditioning or current-KV projections remain explicit costs; these expressions do not imply that the entire block scales by the anchor ratio.

For tiles of at most $B_p$ spatial positions, batched anchor FFTs and skipped-frame inverse FFTs cost $O(N_{\mathrm{cur}}SD\log B_p)$. Pooled scoring costs $O(K_fMD)$, while $L_s$ directional refinement passes over fixed candidate sets of size $M_c$ cost $O(L_sK_fM_c)$ plus Top-$K$ selection. Memory-context processing, the dense warm-up block, KV storage, control scoring, and tensor gathering remain in the total. Historical sparsity reduces attention reads without automatically shrinking the resident KV cache. Boundary extrapolation reuses the first or last adjacent-anchor pair and adds no directly computed visual residuals; unlike residual copying, it incurs spectral reconstruction work for boundary targets. The anchor budget and asymptotic reconstruction cost are unchanged, but measured latency need not be identical. Net wall-clock gains therefore require the runtime accounting in Section~\ref{sec:efficiency}; operation-count savings alone do not establish interactive speed.

\section{Evaluation Details and Additional Analyses}
\label{app:evaluation}
\subsection{Implementation and Evaluation Record}
\label{app:config}
\enlargethispage{\baselineskip}
Qualitative examples use Matrix sample \texttt{fr005\_006} and HY sample \texttt{ts000\_006}.

Tables~\ref{tab:config} and~\ref{tab:config_cast} summarize backbone, CAST, and evaluation settings. Backbone identities and chunk conventions follow the \href{https://github.com/SkyworkAI/Matrix-Game/blob/main/Matrix-Game-3/README.md}{Matrix release}, its \href{https://huggingface.co/Skywork/Matrix-Game-3.0/blob/main/base_model/config.json}{model configuration}, and the \href{https://github.com/Tencent-Hunyuan/HY-WorldPlay/blob/main/README.md}{HY release}. The tables specify hooks, hardware, numerical settings, and evaluators; timing follows Section~\ref{sec:efficiency}. We use the released Light Interaction block-sparse backend with CAST's historical-block selection policy.

\paragraph{Native-chunk compatibility.}
Matrix uses ten current latent frames and dynamically selects five anchors at $r_f=0.50$; HY uses four and dynamically selects two. Both backbones select over the entire current chunk without mandatory endpoints. The anchor set is shared across eligible visual residual sites within each transformer block. Matrix budget experiments use the discrete counts $K_f\in\{3,5,6\}$, corresponding to $r_f\in\{0.30,0.50,0.60\}$.

\paragraph{Benchmark scope and provenance.}
WorldMark uses 500 samples for every method on both backbones. The Light Interaction-style evaluation uses the same 400 trajectories for every method: 200 left--right and 200 forward--backward. All baselines, including Native and Sol Engine, are re-run in our unified environment; the tables contain no scores imported from published result tables. Within each backbone, methods share the evaluator, per-rollout hardware, and timing scope. Results are point estimates without confidence intervals. Speedups in Table~\ref{tab:worldmark} come from our matched runtime evaluation in Table~\ref{tab:lightinteraction}, not the WorldMark trajectories. Non-native rankings exclude Native.

\paragraph{Controls and accounting.}
Paired runs share initial observations, controls, and seeds; calibration and evaluation trajectories are disjoint. Matched-component ablations fix $K_f$, $K_{\mathrm{KV}}$, block layouts, candidate sets, and dense-block schedules. Matrix uses one RTX A6000 per rollout; HY uses two for model/tensor parallel inference. The eight-GPU RTX A6000 server parallelizes independent runs on unused devices. For each of three seeds, timing uses ten warm-up and fifty measured chunks with CUDA synchronization. Mean steady-state latency includes scoring, routing, reconstruction, and retrieval, but excludes VAE decoding and the initial startup phase.

WorldMark's thirteen dimensions comprise translation/rotation variants of direction accuracy, direction purity, motion stability, and response latency; local, global, and revisit memory; and perceptual and aesthetic quality. The Light Interaction protocol reports PSNR, SSIM, and LPIPS against paired native rollouts and under self-comparison, together with VBench. Agreement with Native measures approximation fidelity rather than ground-truth correctness.

\begin{table}[!htbp]
\caption{Backbone and execution configuration ($8\times$ RTX A6000 server).}
\label{tab:config}
\centering\footnotesize
\setlength{\tabcolsep}{3pt}
\renewcommand{\arraystretch}{1.15}
\begin{tabularx}{\linewidth}{@{}>{\raggedright\arraybackslash}p{2.8cm}YY@{}}
\toprule
Field & Matrix-Game 3.0 & HY-World 1.5\\
\midrule
Checkpoint and code provenance & \texttt{Skywork/Matrix-Game-3.0}, \texttt{base\_distilled\_model}; released code with CAST integration & \texttt{tencent/HY-WorldPlay}, \texttt{ar\_distilled\_action\_model}; HunyuanVideo 8B path\\
GPU model and count & $1\times$ RTX A6000 per rollout; batch 1 & $2\times$ RTX A6000 per rollout; model/tensor parallel inference; batch 1\\
Software and attention backend & Python 3.12; PyTorch 2.6; CUDA 12.4; Triton 3.2; FlashAttention 2.7 for dense paths; released Light Interaction block-sparse execution backend & Python 3.10; PyTorch 2.6; CUDA 12.4; Triton 3.2; FlashAttention 2.7 for dense paths; released Light Interaction block-sparse execution backend\\
Model / attention / FFT precision & BF16 model and QKV; FP32 softmax and pooled reductions; FP32 / complex64 FFT; no INT8/FP8 & BF16 model and QKV; FP32 softmax and pooled reductions; FP32 / complex64 FFT; no FP8\\
Resolution and native chunk & $704\times1280$; steady-state 40 decoded / 10 new latent frames; separate 57-frame startup & $480\times832$; steady-state 16 decoded / 4 latent frames; first chunk has 13 decoded frames\\
Denoising schedule & 3 evaluations; retain the released distilled scheduler and its native timestep ordering & 4 evaluations; native few-step AR scheduler; no denoising cache\\
History and KV block layout & Keep native camera-aware retrieval and temporal context; blocks $(t,h,w)=(1,8,8)$ in token coordinates; mask partial blocks & Keep native reconstituted context memory and causal permissions; blocks $(1,8,8)$ in token coordinates; mask partial blocks\\
Control-score hooks / skip scope & Dense control branches where present (released action blocks 0--14); skip only current visual attention/FFN residuals; use debt-only selection where no fresh control branch exists & Dense native action-conditioning path at verified interfaces; skip only current visual attention/FFN residuals; never skip text or conditioning updates\\
Current-frame KV construction & Project K/V for all current visual tokens, retain all permitted current/text KV; sparse Q computation only for selected frames & Project K/V for all current visual tokens; preserve native text/visual masks and permitted current/text KV\\
Dense / sparse schedule & First transformer block dense at every denoising evaluation; all later eligible visual residuals sparse; reset debt per evaluation & First transformer block dense at every denoising evaluation; all later eligible visual residuals sparse; reset debt per evaluation\\
\bottomrule
\end{tabularx}
\end{table}

\begin{table}[!htbp]
\caption{CAST and evaluation settings for both backbones under the shared comparison protocol.}
\label{tab:config_cast}
\centering\footnotesize
\setlength{\tabcolsep}{4pt}
\renewcommand{\arraystretch}{1.15}
\begin{tabularx}{\linewidth}{@{}>{\raggedright\arraybackslash}p{2.8cm}YY@{}}
\toprule
Field & Matrix-Game 3.0 & HY-World 1.5\\
\midrule
$r_f$, $K_f$, $\lambda_{\mathrm{cov}}$ & $0.50$, $5$ of $10$, $0.50$; dynamic Top-5 over all current frames; boundary extrapolation uses the two nearest anchors & $0.50$, $2$ of $4$, $0.50$; dynamic Top-2 over all current frames; boundary extrapolation uses the two nearest anchors\\
Tiles and reliable spectral bins & $8\times8$ spatial token tiles, disjoint; $W_p,|Z_p|>10^{-6}\max W_p+10^{-8}$; FP32 & $8\times8$ spatial token tiles, disjoint; $W_p,|Z_p|>10^{-6}\max W_p+10^{-8}$; FP32\\
$\tau_{\mathrm{low}},\tau_{\mathrm{high}},q,T$ & $(0.25,0.75,2,0.10)$; principal phase branch; DC/Nyquist transport disabled & $(0.25,0.75,2,0.10)$; principal phase branch; DC/Nyquist transport disabled\\
$K_{\mathrm{KV}}$, $r_{\mathrm{KV}}$, $M_c$ & $K_{\mathrm{KV}}=\lceil0.20M\rceil$; target $r_{\mathrm{KV}}=0.20$; $M_c=M$ (full candidate pool); bypass if $M=0$ & $K_{\mathrm{KV}}=\lceil0.20M\rceil$; target $r_{\mathrm{KV}}=0.20$; $M_c=M$; report realized block fraction\\
Projection $P$ and scaling & $P=W_O^{(h)}$, native attention output-projection head slice; no additional learned/sketched projection; FP32 pooled reductions & Same native output-projection head slice and FP32 reductions; squared Euclidean norm\\
Routing weights / stopping & $(\lambda_{\mathrm{anc}},\lambda_{\mathrm{rec}})=(1,1)$; $L_s=2$ directional passes; accept $\Delta J>10^{-6}\max(J,10^{-8})$ & $(\lambda_{\mathrm{anc}},\lambda_{\mathrm{rec}})=(1,1)$; $L_s=2$ directional passes; same relative tolerance\\
Calibration / evaluation / seeds & 32 disjoint calibration trajectories; retain main-text 500 WorldMark / 400 LI counts; paired seed 42; sensitivity seeds 0, 1, 2 & 32 disjoint calibration trajectories; retain main-text 500 WorldMark samples per method and 400 LI counts; same seed plan\\
Trajectories and horizons & Preserve benchmark controls; 12-iteration diagnostic rollout: $57+11\times40=497$ decoded frames; startup excluded from steady-state timing & Preserve benchmark controls; diagnostic horizons 125, 253, 509 decoded frames (32, 64, 128 latent frames)\\
Warm-up and repetitions & 10 unmeasured chunks; 50 timed chunks per seed; 3 seeds; CUDA synchronization; mean steady-state latency & Same repetitions and timing scope; scoring, routing, reconstruction, and retrieval included; VAE decoding excluded; startup excluded\\
Quality/control evaluators & Official VBench and WorldMark protocols; LPIPS-AlexNet; RGB $[0,1]$ for PSNR/SSIM & Same evaluators, revisions, and preprocessing; standalone VGGT diagnostic is separate from WorldMark\\
Uniform-anchor controlled baseline & Table~\ref{tab:ablation}, Variant A: uniform/interleaved anchors, linear residual recovery, independent KV routing & Same controlled ablation definition\\
Sol Engine & Official publicly released implementation, fully re-run in our unified environment & Official publicly released implementation, fully re-run with the same evaluator, hardware, and timing scope as CAST\\
\bottomrule
\end{tabularx}
\end{table}
\FloatBarrier

\subsection{Diagnostic Measurements}
\label{app:diagnostics}
\paragraph{Oracle frame difficulty.}
\looseness=-1 The frame-difficulty diagnostic replays identical block inputs through the dense reference. For an eligible interior frame, its neighboring anchors are fixed, its direct residual computation is omitted, and Eq.~\eqref{eq:sensitivity} measures the resulting error. Figure~\ref{fig:empirical_motivation} compares sensitivity at block $\ell$ with error at $\ell+1$, matching the deployed prediction lag. This is an association diagnostic rather than a causal guarantee for frame selection.

\paragraph{Oracle corrections and phase progression.}
Let $F_t$ and $F_t^{\mathrm{lin}}$ be the spatial Fourier transforms of the reference and interpolated residuals. The two diagnostic oracle corrections are
\begin{equation}
F_t^{\mathrm{phase}}=|F_t^{\mathrm{lin}}|e^{\mathrm{i}\Arg F_t},\qquad
F_t^{\mathrm{mag}}=|F_t|e^{\mathrm{i}\Arg F_t^{\mathrm{lin}}}.
\label{eq:oracles}
\end{equation}
These reference-based diagnostics do not run at inference. For coherent constant-velocity translation,
\begin{equation}
R_b(x)=R_a(x-d)\ \Longrightarrow\ F_b(\omega)=F_a(\omega)e^{-\mathrm{i}\omega^\top d},\qquad
F_t(\omega)=F_a(\omega)e^{-\mathrm{i}\alpha\omega^\top d}.
\label{eq:shift}
\end{equation}
This implies linear unwrapped phase at every frequency, not preferential low-frequency linearity. To compare trajectories in Figure~\ref{fig:spectral_phase_motivation}, we fix one coefficient within a channel and spatial region and define
\begin{equation}
p_{t,\omega}=\frac{\tilde\phi_t(\omega)-\tilde\phi_a(\omega)}{\tilde\phi_b(\omega)-\tilde\phi_a(\omega)}.
\label{eq:phase_progress}
\end{equation}
\looseness=-1 Uniform progression gives $p_{t,\omega}=\alpha$, with endpoints fixed to zero and one. Negative values indicate motion opposite to the net endpoint phase displacement. Weak-amplitude coefficients, near-zero endpoint phase differences, and inconsistent phase branches can distort this normalization. Temporal comparisons in Figure~\ref{fig:empirical_motivation} independently rescale the two curves and preserve the $\ell$-to-$\ell+1$ lag.

\paragraph{Spectral decomposition.}
For radial band $\mathcal B$, the absolute error-energy fraction is
\begin{equation}
\eta(\mathcal B)=\frac{\sum_{p,c,\omega\in\mathcal B}\kappa_\omega|F_{t,p,c}-F^{\mathrm{lin}}_{t,p,c}|^2}
{\sum_{p,c,\omega}\kappa_\omega|F_{t,p,c}-F^{\mathrm{lin}}_{t,p,c}|^2+\epsilon}.
\label{eq:bandenergy}
\end{equation}
Oracle phase and magnitude corrections use Eq.~\eqref{eq:oracles} and the same tile/FFT conventions. Phase is undefined for zero-energy spectra. Let $\phi_{i,p,c}=\Arg F_{i,p,c}$; the transport diagnostic measures the energy-weighted circular distance
\begin{equation}
\left|\wrap_{[-\pi,\pi)}\left(\phi_{t,p,c}-\phi_{a,p,c}-\alpha\theta_p\right)\right|.
\label{eq:phaseerror}
\end{equation}

\paragraph{Routing diagnostics.}
With the anchor set, candidate pool, and per-anchor KV budget fixed, Table~\ref{tab:coupling_diagnostic} compares independent omission-importance Top-$K$ masks with coupled masks against a dense-history reference. Anchor perturbations and interval reconstruction errors are measured on identical block inputs, separating the immediate approximation from subsequent rollout drift. The objective is a frozen-transport surrogate; the reported errors measure the resulting residual approximations rather than certify the bound for the complete nonlinear inference process.

\paragraph{Generation quality and controls.}
\looseness=-1 The quality score $Q$ uses the official VBench quality aggregation~\citep{huang2023vbench}. Paired RGB PSNR uses a declared intensity range and the same initial observation, control sequence, and random seed. It measures closeness to the dense generator, which can itself be wrong. Fixed-input residual diagnostics isolate approximation effects from subsequent context drift accumulated over later chunks in free-running rollouts.

The camera-control diagnostic compares inferred and commanded relative rotations in a common coordinate system. The rotation error is the geodesic angle $\arccos(\operatorname{clip}((\operatorname{tr}(R_{\mathrm{cmd}}^\top R_{\mathrm{est}})-1)/2,-1,1))$. Translation-direction error is reported only for commands with nonzero movement. The complementary camera-pose diagnostic uses \href{https://github.com/facebookresearch/vggt}{VGGT-1B}~\citep{wang2025vggt}. Its OpenCV world-to-camera matrices are converted to camera-to-world poses, commanded poses are transformed to the same axes, and both trajectories are expressed relative to the first camera. Relative rotations and unit translation directions are evaluated in degrees, without fitting a similarity transform to each generated rollout. Translation-direction scoring requires finite poses, a proper orthonormal rotation (tolerance $10^{-3}$), and nonzero translation norm ($>10^{-6}$ in normalized coordinates). The valid-pose rates are $94.6\%$ (Matrix) and $96.2\%$ (HY). This diagnostic complements the official WorldMark evaluation. Visual quality alone does not establish discrete action consistency. Revisit consistency assesses accumulated drift. These protocols assess control adherence and accumulated drift alongside the benchmark quality metrics.

\paragraph{Latency measurement.}
We report mean steady-state chunk latency over fifty timed chunks per seed and three seeds, following ten unmeasured warm-up chunks. CUDA synchronization brackets the measured execution, including retrieval, scoring, routing, and reconstruction while excluding VAE decoding; startup is excluded. Speedup divides matched Native latency by method latency at identical resolution, chunk size, and per-backbone hardware. Generation throughput is decoded frames per measured time and does not by itself equal the interaction response rate. Quality comparisons remain point estimates without confidence intervals for the reported evaluation samples.

\subsection{Additional Ablations and Qualitative Comparisons}
\label{app:ablations}

\subsubsection{Controlled Design Alternatives}
The main paper retains the progressive component table; the following one-factor comparisons separate the choice of selector and reconstruction operator under the same end-to-end evaluation.

\paragraph{Controlled frame-selection and reconstruction ablations.}
 Following the standard one-factor-at-a-time organization used in diffusion-acceleration studies, Table~\ref{tab:reconstruction_ablation} changes one design choice while holding the rest of the pipeline fixed. In the first panel, all selectors use \fc{} and coupled routing; in the second, all reconstruction operators use control-aware anchors and coupled routing. Each row is evaluated end-to-end with the same metrics, rather than mixing selector-only and reconstruction-only diagnostics with different evaluation targets in one set of columns.
 
 \begin{table}[H]
 \caption{Controlled frame-selection and reconstruction ablations on Matrix-Game 3.0. Every row uses $r_f=0.50$, $r_{\mathrm{KV}}=0.20$, and the same end-to-end evaluation. T-Motion, Local, Global, and Revisit are representative WorldMark scores; all are higher-is-better.}
 \label{tab:reconstruction_ablation}
 \centering\small\renewcommand{\arraystretch}{1.15}
 \setlength{\tabcolsep}{2.5pt}
 \begin{tabular*}{\linewidth}{@{\extracolsep{\fill}}lrrrrrr@{}}
 \toprule
 & & & \multicolumn{4}{c}{WorldMark$\uparrow$}\\
 \cmidrule(lr){4-7}
 Variant & VBench$\uparrow$ & LPIPS$\downarrow$ & T-Motion & Local & Global & Revisit\\
 \midrule
 \multicolumn{7}{@{}l}{\textit{Frame selection; \fc{} and coupled routing fixed}}\\
 Uniform/interleaved & 0.7437 & 0.4884 & 76.914 & 73.478 & 42.806 & 77.138\\
 Latent-motion magnitude & 0.7506 & 0.4769 & 79.312 & 75.087 & 44.578 & 79.341\\
 Control sensitivity & \textbf{0.7570} & \textbf{0.4657} & \textbf{79.830} & \textbf{76.531} & \textbf{46.645} & \textbf{81.514}\\
 \midrule
 \multicolumn{7}{@{}l}{\textit{Reconstruction; control-aware selection and coupled routing fixed}}\\
 Linear interpolation & 0.7392 & 0.4948 & 77.103 & 73.337 & 42.164 & 76.782\\
 Phase transport without confidence & 0.7495 & 0.4781 & 78.864 & 75.008 & 44.612 & 79.203\\
 Full \fc{} & \textbf{0.7570} & \textbf{0.4657} & \textbf{79.830} & \textbf{76.531} & \textbf{46.645} & \textbf{81.514}\\
 \bottomrule
 \end{tabular*}
 \end{table}
 
 With reconstruction and routing fixed, control sensitivity improves VBench by 0.0133 over uniform/interleaved selection, reduces LPIPS by 0.0227, and raises global and revisit memory by 3.839 and 4.376 points. Latent-motion magnitude is consistently intermediate, suggesting that generic motion is useful but does not fully capture action-conditioned frame importance. With selection and routing fixed, full \fc{} improves VBench by 0.0178 over linear interpolation, reduces LPIPS by 0.0291, and gains 4.481 global-memory and 4.732 revisit-memory points. Phase transport without confidence recovers part of this gap, while the remaining gain from frequency-dependent confidence is visible across both visual and memory metrics. Residual-domain measurements are reported separately in Appendix~\ref{app:diagnostics}; these tables use uniform end-to-end metrics for direct comparison.

\subsubsection{Fixed-Input Coupling Diagnostics}
\paragraph{Does reconstruction coupling add value?}
 We compare C and D on identical block inputs with the same control-aware anchors, historical candidate pool, and per-anchor KV cardinality. Independent routing retains the Top-$K$ blocks under the local omission score $u_{i,m}$; coupled routing starts from those masks and applies one forward and one backward directional refinement pass using the reconstruction-coupled objective in Eq.~\eqref{eq:objective}. This changes which historical blocks are retained without changing the 20\% historical KV retention budget at any anchor.
 
 \begin{table}[H]
 \caption{Fixed-input coupling ablation on Matrix-Game 3.0. Anchor and interval errors are normalized residual errors under identical anchors and KV budgets. Local, Global, and Revisit are WorldMark memory scores measured on the paired rollouts for both routing variants.}
 \label{tab:coupling_diagnostic}
 \centering\small\renewcommand{\arraystretch}{1.15}
 \setlength{\tabcolsep}{3pt}
 \begin{tabular*}{\linewidth}{@{\extracolsep{\fill}}lrrrrrr@{}}
 \toprule
 KV routing & Anchor error$\downarrow$ & Interval error$\downarrow$ & VBench$\uparrow$ & Local$\uparrow$ & Global$\uparrow$ & Revisit$\uparrow$\\
 \midrule
 Independent (C) & 0.0831 & 0.1127 & 0.7479 & 75.046 & 44.387 & 79.206\\
 Coupled, $L_s=2$ (D) & \textbf{0.0774} & \textbf{0.1013} & \textbf{0.7570} & \textbf{76.531} & \textbf{46.645} & \textbf{81.514}\\
 \bottomrule
 \end{tabular*}
 \end{table}
 
\looseness=-1 Coupling reduces anchor-local error by 6.9\% and interval reconstruction error by 10.1\%. The larger reduction between anchors is accompanied by gains of 1.485, 2.258, and 2.308 points in local, global, and revisit memory. This pattern supports the reconstruction-risk objective: coupled routing selects historical evidence that is useful both to each endpoint and to the frames reconstructed between adjacent endpoints. C and D share the same budget, ruling out additional KV capacity as the cause.

\subsubsection{Budget and Refinement Sensitivity}
\label{app:budget_sweep}
\paragraph{Sensitivity to computation budgets.}
 We vary the retained-frame ratio $r_f=K_f/N_{\mathrm{cur}}$ and retained-history ratio $r_{\mathrm{KV}}=K_{\mathrm{KV}}/M$ one at a time on Matrix-Game 3.0, while fixing $L_s=2$. Table~\ref{tab:budget_sensitivity} reports both visual quality and the WorldMark global and revisit memory scores most directly affected by current-frame and historical-context budgets in these comparisons.
 
 \begin{table}[H]
 \caption{Budget sensitivity on Matrix-Game 3.0 with two directional refinement passes. Ratios are retained fractions; Global and Revisit are WorldMark memory scores, and latency is milliseconds per chunk. All settings use identical prompts, actions, and paired random seeds.}
 \label{tab:budget_sensitivity}
 \centering\small\renewcommand{\arraystretch}{1.15}
 \setlength{\tabcolsep}{2.7pt}
 \begin{tabular*}{\linewidth}{@{\extracolsep{\fill}}ccrrrrrr@{}}
 \toprule
 & & & & \multicolumn{2}{c}{WorldMark$\uparrow$} & & \\
 \cmidrule(lr){5-6}
 $r_f$ & $r_{\mathrm{KV}}$ & VBench$\uparrow$ & LPIPS$\downarrow$ & Global & Revisit & Latency$\downarrow$ & Speedup$\uparrow$\\
 \midrule
 0.30 & 0.20 & 0.7441 & 0.4936 & 42.738 & 77.246 & 1,717 & $2.52\times$\\
 0.50 & 0.10 & 0.7512 & 0.4788 & 42.911 & 77.926 & 1,909 & $2.26\times$\\
 \textbf{0.50} & \textbf{0.20} & \textbf{0.7570} & \textbf{0.4657} & \textbf{46.645} & \textbf{81.514} & \textbf{2,007} & $\mathbf{2.15\times}$\\
 0.50 & 0.30 & 0.7584 & 0.4608 & 47.386 & 82.263 & 2,078 & $2.08\times$\\
 0.60 & 0.20 & 0.7601 & 0.4529 & 47.219 & 82.487 & 2,363 & $1.83\times$\\
 \bottomrule
 \end{tabular*}
 \end{table}
 
 Reducing $r_f$ from 0.50 to 0.30 saves 290 ms but lowers VBench by 0.0129 and reduces global and revisit memory by 3.907 and 4.268 points. Raising $r_f$ to 0.60 provides only 0.0031 additional VBench and less than one point of memory improvement, while reducing speedup from $2.15\times$ to $1.83\times$. Historical context shows a similar saturation pattern. Retaining 10\% rather than 20\% of historical KV blocks loses 3.734 global-memory and 3.588 revisit-memory points, whereas increasing the ratio from 20\% to 30\% adds 71 ms for gains of only 0.741 and 0.749 points. We therefore use $(r_f,r_{\mathrm{KV}})=(0.50,0.20)$ as the default operating point.

\paragraph{Number of directional refinement passes.}
 We finally vary $L_s$ at $(r_f,r_{\mathrm{KV}})=(0.50,0.20)$. Here, $L_s=0$ uses the independent per-anchor Top-$K$ initialization, $L_s=1$ applies one forward directional refinement pass, $L_s=2$ applies one forward followed by one backward directional refinement pass, and $L_s=4$ repeats this pair twice. Each pass visits all anchors once, alternating forward and backward directions across the temporally ordered anchor sequence.
 
 \begin{table}[H]
 \caption{Directional refinement sensitivity on Matrix-Game 3.0 at $r_f=0.50$ and $r_{\mathrm{KV}}=0.20$. Global and Revisit are WorldMark memory scores; interval NRMSE is measured against fully computed residuals, and latency is milliseconds per chunk.}
 \label{tab:sweep_sensitivity}
 \centering\small\renewcommand{\arraystretch}{1.15}
 \setlength{\tabcolsep}{2.5pt}
 \begin{tabular*}{\linewidth}{@{\extracolsep{\fill}}crrrrrrr@{}}
 \toprule
 & & & \multicolumn{2}{c}{WorldMark$\uparrow$} & & & \\
 \cmidrule(lr){4-5}
 $L_s$ & VBench$\uparrow$ & LPIPS$\downarrow$ & Global & Revisit & Interval NRMSE$\downarrow$ & Latency$\downarrow$ & Speedup$\uparrow$\\
 \midrule
 0 & 0.7479 & 0.4776 & 44.387 & 79.206 & 0.1127 & 1,976 & $2.19\times$\\
 1 & 0.7549 & 0.4689 & 45.916 & 80.721 & 0.1058 & 1,991 & $2.17\times$\\
 \textbf{2} & \textbf{0.7570} & \textbf{0.4657} & \textbf{46.645} & \textbf{81.514} & \textbf{0.1013} & \textbf{2,007} & $\mathbf{2.15\times}$\\
 4 & 0.7574 & 0.4649 & 46.781 & 81.672 & 0.1008 & 2,038 & $2.12\times$\\
 \bottomrule
 \end{tabular*}
 \end{table}
 
 The first two directional refinement passes recover most of the quality and memory lost by independent routing. Moving from $L_s=0$ to $L_s=2$ raises VBench by 0.0091, improves global and revisit memory by 2.258 and 2.308 points, and reduces interval NRMSE by 10.1\%, for 31 ms of additional latency. Increasing $L_s$ from 2 to 4 yields only 0.0004 VBench, 0.136 global-memory points, 0.158 revisit-memory points, and a 0.5\% relative interval-error improvement while reducing speedup to $2.12\times$. This saturation supports the fixed two-pass schedule used in the main experiments.

\subsubsection{Additional Qualitative Comparisons}
Figures~\ref{fig:appendix_qualitative}--\ref{fig:appendix_qualitative_5} extend Figure~\ref{fig:qualitative_baselines} with five additional samples per backbone. Each panel compares Native, SVG, TeaCache, Light Interaction (LI), and \method{} at matched frame indices under the same control sequence. These selected rollouts illustrate scene evolution and action reversals; aggregate comparisons are reported in the main experiments.

\FloatBarrier
\setlength{\floatsep}{3pt plus 1pt minus 1pt}
\setlength{\textfloatsep}{3pt plus 1pt minus 1pt}
\captionsetup[figure]{font=footnotesize,skip=2pt}
\begin{figure}[H]
\centering
\includegraphics[width=\textwidth]{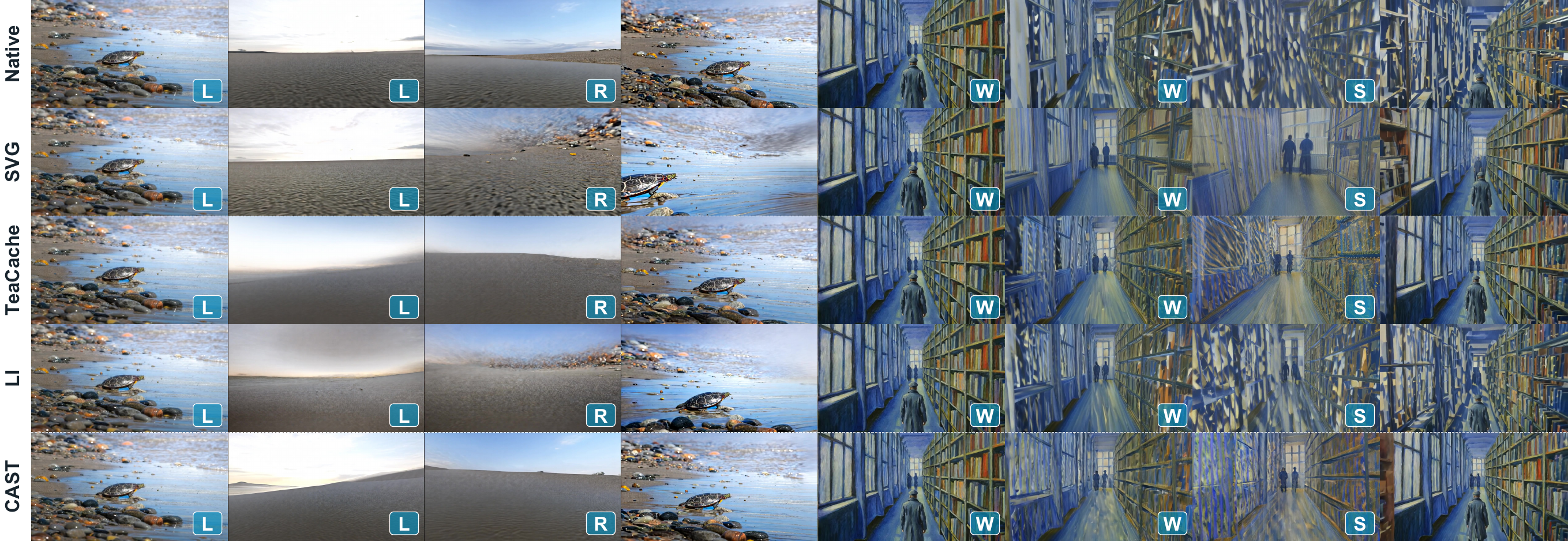}
\caption{WorldMark rollouts: Matrix-Game 3.0, \texttt{fr012\_007} (L/R, left), and HY-World 1.5, \texttt{ts018\_006} (W/S, right). All methods use matched frame indices. Columns show the initial view, the midpoint and end of the first action, and the end of the second action, including its reversal of the initial command.}
\label{fig:appendix_qualitative}
\end{figure}

\begin{figure}[H]
\centering
\includegraphics[width=\textwidth]{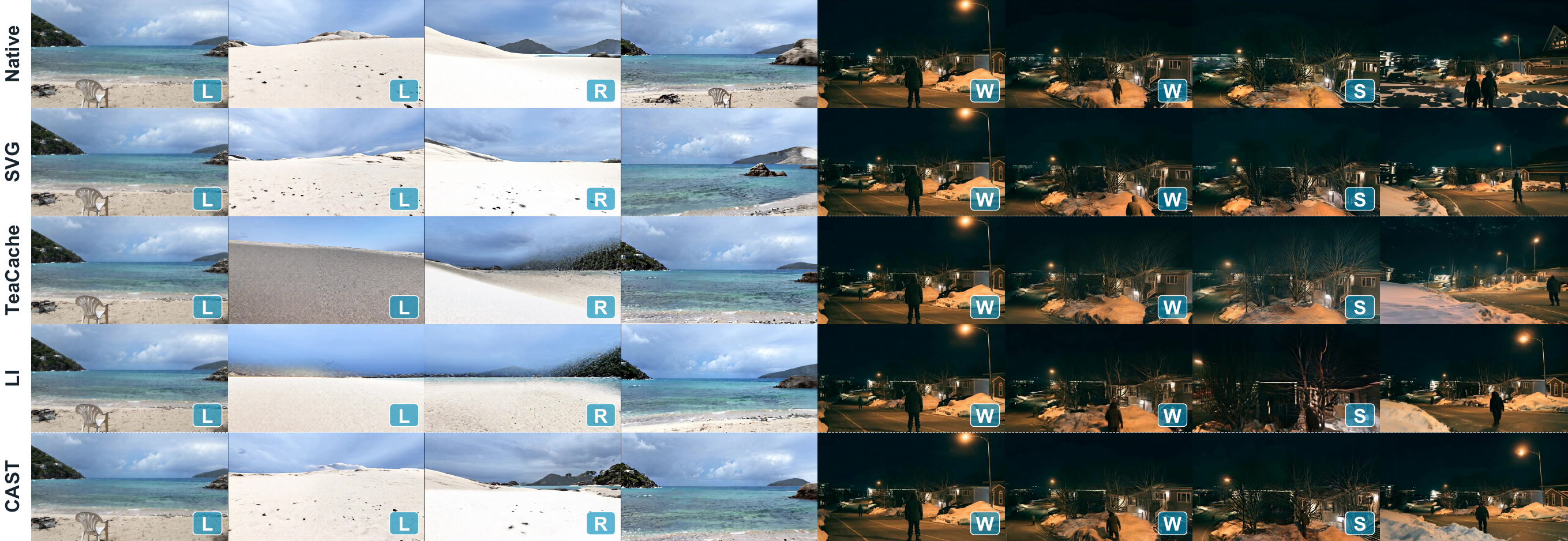}
\caption{WorldMark rollouts: Matrix-Game 3.0, \texttt{fr019\_007} (L/R, left), and HY-World 1.5, \texttt{tr006\_006} (W/S, right). Methods and frame selection follow Figure~\ref{fig:appendix_qualitative}.}
\label{fig:appendix_qualitative_2}
\end{figure}

\begin{figure}[H]
\centering
\includegraphics[width=\textwidth]{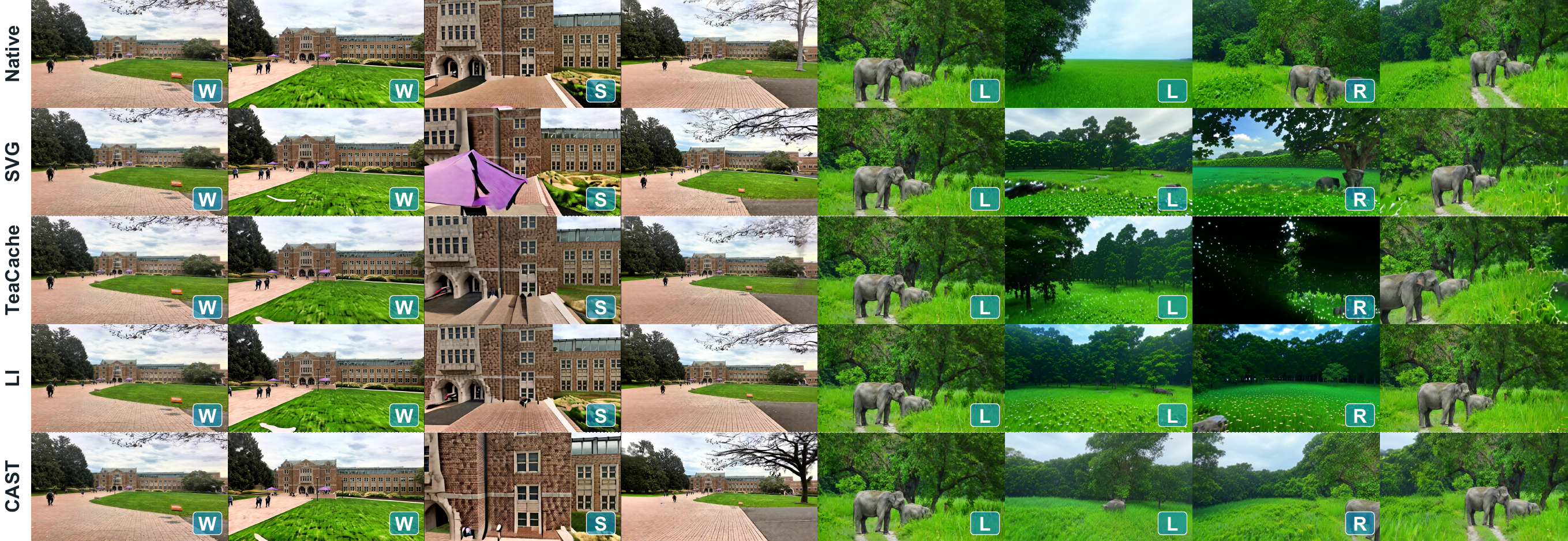}
\caption{WorldMark rollouts: Matrix-Game 3.0, \texttt{fr007\_006} (W/S, left), and HY-World 1.5, \texttt{fr018\_007} (L/R, right). Methods and frame selection follow Figure~\ref{fig:appendix_qualitative}.}
\label{fig:appendix_qualitative_3}
\end{figure}

\begin{figure}[H]
\centering
\includegraphics[width=\textwidth]{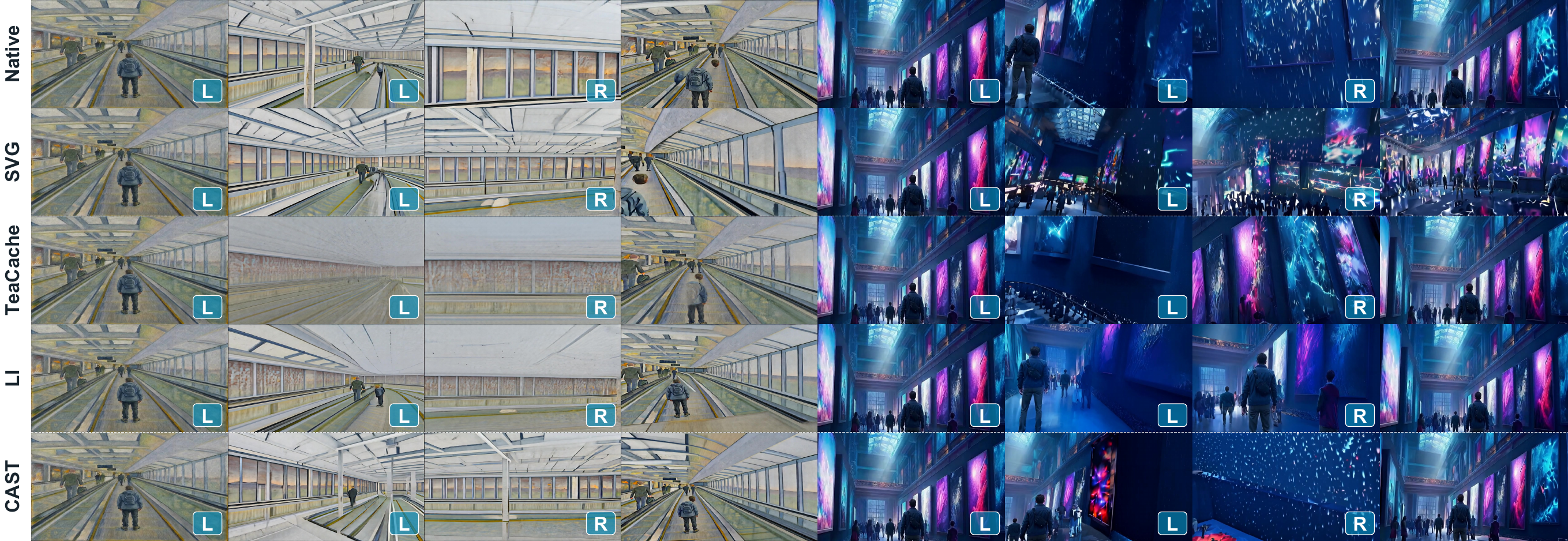}
\caption{WorldMark rollouts: Matrix-Game 3.0, \texttt{ts020\_007} (L/R, left), and HY-World 1.5, \texttt{ts024\_007} (L/R, right). Methods and frame selection follow Figure~\ref{fig:appendix_qualitative}.}
\label{fig:appendix_qualitative_4}
\end{figure}

\begin{figure}[H]
\centering
\includegraphics[width=\textwidth]{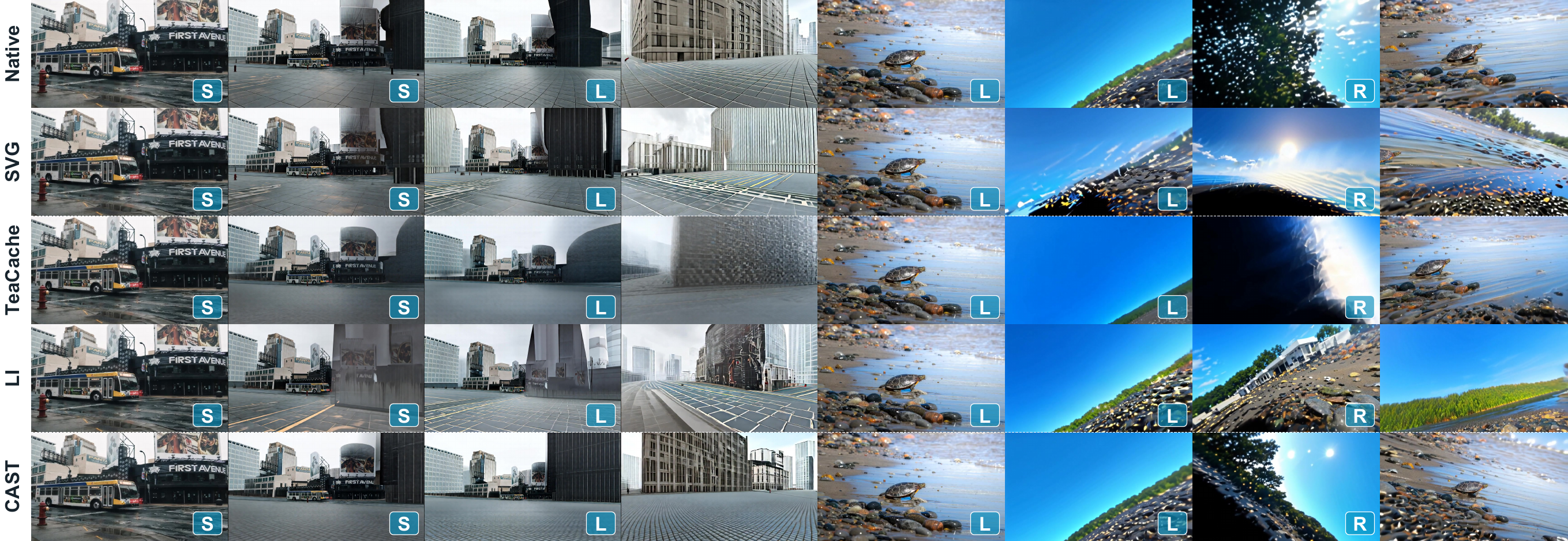}
\caption{WorldMark rollouts: Matrix-Game 3.0, \texttt{fr003\_010} (S/L, left), and HY-World 1.5, \texttt{fr012\_007} (L/R, right). Methods and frame selection follow Figure~\ref{fig:appendix_qualitative}.}
\label{fig:appendix_qualitative_5}
\end{figure}

\FloatBarrier

\end{document}